\documentclass{article}

\usepackage{microtype}
\usepackage{graphicx}
\usepackage{subfigure}
\usepackage{booktabs} 

\usepackage{hyperref}
\usepackage{xcolor}
\usepackage{pifont}
\usepackage{amsmath}
\usepackage{amsthm}
\usepackage{ulem}
\usepackage[colorinlistoftodos]{todonotes}
\usepackage{marginnote}
\usepackage{amsmath,amssymb}
\usepackage{lipsum}
\usepackage{multicol}
\usepackage{adjustbox}
\usepackage{stackrel}
\usepackage[para]{threeparttable}
\usepackage{tikz}  
\usepackage{pgfplots}
\usepackage{pgfplotstable}
\usepackage{filecontents}
\usepackage{verbatim}
\pgfplotsset{compat=1.14}
\usetikzlibrary{shapes,decorations,shadows}  
\usetikzlibrary{decorations.shapes}  
\usetikzlibrary{patterns}  
\usetikzlibrary{intersections}
\usepackage{xspace}
\usepgfplotslibrary{fillbetween}
\newcommand{\colormatern}{yellow}
\newcommand{\colorspearmint}{violet}
\newcommand{\colorspherewarpingorigin}{red}

\newcommand{\addref}[1]{\textcolor{red}{[\textbf{REF}]}}
\newcommand{\ie}[1]{\textit{i.e.}\xspace}
\newcommand{\eg}[1]{\textit{e.g.}\xspace}

\DeclareMathOperator*{\argmax}{arg\; max}
\DeclareMathOperator*{\argmin}{arg\; min}
\DeclareMathOperator{\x}{\mathbf{x}}
\DeclareMathOperator{\aaa}{\mathbf{a}}

\DeclareMathOperator{\D}{\mathcal{D}}

\DeclareMathOperator{\dd}{\mathbf{b}}

\definecolor{darkgray}{rgb}{0.4,0.4,0.4}

\newtheorem{theorem}{Theorem}

\usepackage[accepted]{icml2018}

\icmltitlerunning{BOCK : Bayesian Optimization with Cylindrical Kernels}

\begin{document}

\twocolumn[
\icmltitle{BOCK : Bayesian Optimization with Cylindrical Kernels}




\begin{icmlauthorlist}
\icmlauthor{ChangYong Oh}{equal}
\icmlauthor{Efstratios Gavves}{equal}
\icmlauthor{Max Welling}{equal,cifar}
\end{icmlauthorlist}

\icmlaffiliation{equal}{QUvA Lab, Informatic Institute, University of Amsterdam, Amsterdam, Netherlands}
\icmlaffiliation{cifar}{Canadian Institute for Advanced Research, Toronto, Canada}

\icmlcorrespondingauthor{ChangYong Oh}{changyong.oh0224@gmail.com}
\icmlcorrespondingauthor{Efstratios Gavves}{efstratios.gavves@gmail.com}
\icmlcorrespondingauthor{Max Welling}{m.welling@uva.nl}

\icmlkeywords{Bayesian Optimization, Hyperparameter Optimization, Gaussian Process, Kernel, Geometry, Warping, Radial kernel, High dimensional spaces, Change of coordinates, Transformation, Resolution, Boundary issue}

\vskip 0.3in
]



\printAffiliationsAndNotice{}  

\begin{abstract}

A major challenge in Bayesian Optimization is the \textit{boundary issue}~\cite{swersky2017improving} where an algorithm spends too many evaluations near the boundary of its search space.
In this paper we propose BOCK, Bayesian Optimization with Cylindrical Kernels, whose basic idea is to transform the ball geometry of the search space using a cylindrical transformation.
Because of the transformed geometry, the Gaussian Process-based surrogate model spends less budget searching near the boundary, while concentrating its efforts relatively more near the center of the search region, where we expect the solution to be located. 
We evaluate BOCK extensively, showing that it is not only more accurate and efficient, but it also scales successfully to problems with a dimensionality as high as 500. 
We show that the better accuracy and scalability of BOCK even allows optimizing modestly sized neural network layers, as well as neural network hyperparameters.

\end{abstract}

\section{Introduction}
When we talk about stars and galaxies we use \textit{parsecs} to describe structures, yet when we discuss the world around us we use \textit{meters}. In other words, the natural lengthscale scale with which we describe the world increases with distance away from us. We believe this same idea is useful when performing optimization in high dimensional spaces.

In Bayesian Optimization (or other forms of hyperparameter optimization) we define a cube or a ball and search for the solution inside that volume. The origin of that sphere is special in the sense that this represents the part of space with the highest probability if finding the solution. Moreover, in high dimensions, when we move outwards, the amount of volume contained in an annulus with width $\delta R$, $A(\mathbf{c}; R-\delta R ,R) = \{\mathbf{x} \vert R -\delta R < \Vert \mathbf{x} -\mathbf{c} \Vert < R\}$, grows exponentially with distance $R$. As such, if we would spend an equal amount of time searching each volume element $\delta V$, we would spend all our time at the boundary of our search region. This effective attraction to the places with more volume is the equivalent of an "entropic force" in physics, and in the case of optimization is highly undesirable, since we expect the solution at a small radius $R$. 

\begin{figure}[t!]
\vskip 0.2in
\begin{center}
\centerline{\includegraphics[width=1.1\columnwidth]{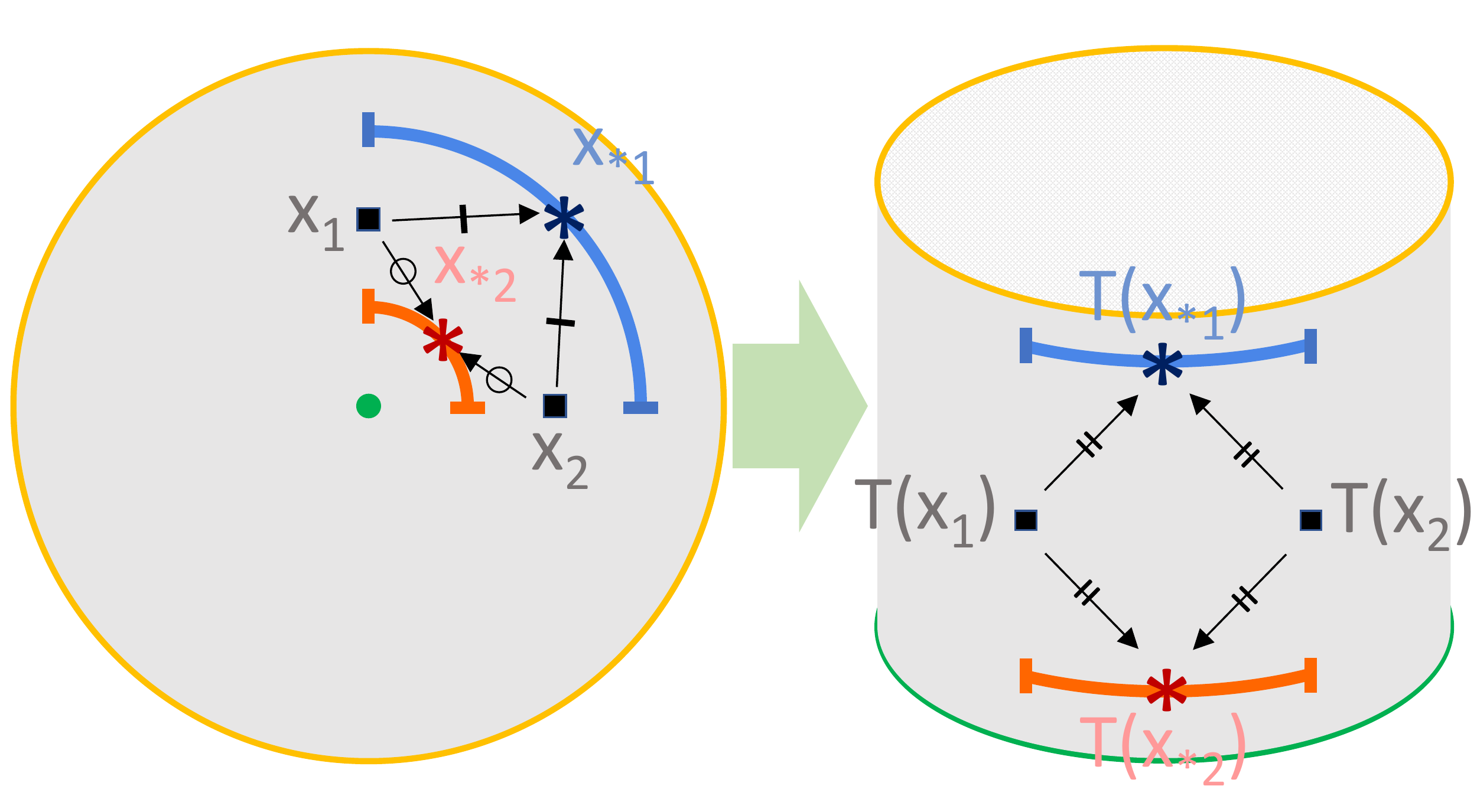}}
\caption{
Many of the problems in Bayesian Optimization relate to the \textit{boundary issue} (too much volume near the boundary~\cite{swersky2017improving}), because of the Euclidean geometry of the search space ball.
Because of the boundary issue, we spend much of the evaluation budget in a particular region of the search space, the boundaries, which contradicts our  \textit{prior assumption} that the solution most likely lies close to the origin. We propose BOCK, whose basic idea is to apply a cylindrical transformation of the search space that expands the volume near the ball center while contracting the volume near the boundaries.
}
\label{figure:transformation}
\end{center}
\vskip -0.2in
\end{figure}

In this paper we, therefore, reformulate Bayesian Optimization in a transformed space, where a ball, $B(\mathbf{x}; R) = \{\mathbf{x} \vert \Vert \mathbf{x} -\mathbf{c} \Vert \le R\}$, is mapped to a cylinder, $C(p,q;\mathbf{c},L = \{(r, \mathbf{a}) \Vert r \in [p, q], \Vert \mathbf{d} - \mathbf{c} \Vert = L \}$ (see Figure \ref{figure:transformation}). In this way, every annulus of width $\delta R$ contains an equal amount of volume for every radius $R$, and the entropic force pulling the optimizer to the boundary disappears. We call our method BOCK, for \textit{Bayesian Optimization with Cylindrical Kernel}. We find that our algorithm is able to successfully handle much higher dimensional problems than standard Bayesian optimizers. As a result, we manage to not only optimize modestly sized neural network layers (up to 500 dimensions in our experiments), obtaining solutions competitive to SGD training, but also hyper-optimize stochastic depth Resnets~\cite{huang2016deep}.

\section{Background}
\subsection{Bayesian Optimization}

\begin{algorithm}[tb]
   \caption{Bayesian Optimization}
   \label{alg:bayesian_optimization_workflow}
\begin{algorithmic}[1]
	\STATE {\bfseries Input:} surrogate model $\mathcal{M}$, acquisition function $\alpha$, search space $X$, initial training data $\mathcal{D}_{init}$, function $f$
   	
    \STATE {\bfseries Output:} optimum $\x_{opt} \in X$ of $f$
   	
    \STATE Initialize $\mathcal{D} = \mathcal{D}_{init}$
   	\WHILE{evaluation budget available}
   	
    \STATE Set $\mu(\cdot \vert \mathcal{D}), \sigma^2(\cdot \vert \mathcal{D}) \leftarrow \mathcal{M} \vert \mathcal{D}$ \textcolor{darkgray}{\small{\textit{$//$ Surrogate function returns predictive mean function and predictive variance function by fitting} $\mathcal{M}$ to $\mathcal{D}$}}
   	
    \STATE Maximize $\widehat{\x} =  \displaystyle\argmax_{\x \in X} \alpha(\mu(\x \vert \mathcal{D}),\sigma^2(\x \vert \mathcal{D}))$\\ \textcolor{darkgray}{\small{\textit{$//$ Acquisition function suggests next evaluation by maximization}}}
   	
    \STATE Evaluate $\hat{y}=f(\widehat{\x})$ \textcolor{darkgray}{\small{\textit{$//$ Evaluate the score of the point selected by the acquisition function}}}
    
   	\STATE Set $\mathcal{D} \leftarrow \mathcal{D} \cup \{(\widehat{\x}, \hat{y})\}$     \textcolor{darkgray}{\small{\textit{$//$ Update the training dataset by including the newly evaluated pair $(\widehat{\x}, \hat{y})$}}}
    
   	\ENDWHILE
\end{algorithmic}
\caption{Bayesian Optimization pipeline.}
\end{algorithm}


Bayesian optimization aims at finding the global optimum of black-box functions, namely
%
\begin{equation}\label{eq:optimization}
	\x_{opt}=\argmin_{\x} f(\x)
\end{equation}
The general pipeline of Bayesian Optimization is given in Alg.~\ref{alg:bayesian_optimization_workflow}.
Prior to starting, a search space must be defined, where the optimum $f(\x_{opt})$ will be searched for.
Given this search space, the initial training dataset must be set, typically by naive guessing where the solution might lie or by informed expert knowledge of the problem. Having completed these two steps, Bayesian Optimization proceeds in an iterative fashion.
At each round, in the absence of any other information regarding the nature of $f(\x)$ a surrogate model attempts to approximate the behavior of $f(\x)$ based on the so far observed points $(\x_i, y_i), y_i=f(\x_i)$.
The surrogate function is then followed by an acquisition function that suggests the next most interesting point $\x_{i+1}$ that should be evaluated.
The pair $(\x_i, y_i)$ is added to the training dataset, $\D=\D\cup(\x_i, y_i)$, and the process repeats until the optimization budget is depleted.

The first design choice of the Bayesian Optimization pipeline is the surrogate model.
The task of the surrogate model is to model probabilistically the behavior of $f(\cdot)$ in the $x$-space in terms of \textit{(a)} a predictive mean $\mu(\x_* \vert \D)$ that approximates the value of $f(x)$ at any point $\x_*$, and \textit{(b)} a predictive variance that represents the uncertainty of the surrogate model in this prediction.
Any model that can provide a predictive mean and variance can be used as a surrogate model, including random forests~\cite{hutter2011sequential}, tree-based models~\cite{bergstra2011algorithms} and neural networks~\cite{snoek2015scalable, springenberg2016bayesian}.
Among other things, Gaussian Processes not only provide enough flexibility it terms of kernel design but also allow for principled and tractable quantification of uncertainty~\cite{rasmussen2006gaussian}. Therefore, we choose Gaussian Processes as our surrogate model.
The predictive mean and the predictive variance of Gaussian processes are given as below
\begin{align}
	\mu(\x_* \vert \D) &= K_{*\D} (K_{\D\D} + \sigma^2 I)^{-1} \mathbf{y} \label{eq:predictive_mean} \\ 
	\sigma^2(\x_* \vert \D) &= K_{**} - K_{*\D} (K_{\D\D} + \sigma_{obs}^2 I)^{-1} K_{\D*}\label{eq:predictive_variance}
\end{align}
where $K_{**}=K(\x_*, \x_*)$, $K_{*\D}$ is a row vector whose \textit{i}th entry is $K(\x_*, \x_i)$, $K_{\D*} = (K_{*\D})^T$, $[K_{\mathcal{DD}}]_{i,j} = K(\x_i, \x_j)$, $\sigma_{obs}^2$ is the variance of observational noise and $\D=\{(\x_i, y_i)\}_i$ is the dataset of observations so far.

The second design choice of the Bayesian Optimization pipeline is the acquisition function.
The predictive mean and the predictive variance from the surrogate model is input to the acquisition function that quantifies the significance of every point in $\x$ as a next evaluation point.
While different acquisition functions have been explored in the literature~\cite{thompson1933likelihood, kushner1964new, movckus1975bayesian, srinivas2009gaussian, hennig2012entropy, hernandez2014predictive}, they all share the following property: they return high scores at regions of either high predictive variance (high but uncertain reward), or low predictive mean (modest but certain reward).

Last, the third design choice of the Bayesian Optimization pipeline, often overlooked, is the search space.
In \cite{snoek2014input} the kernel of the surrogate model is defined on a warped search space, thus allowing for a more flexible modeling of $f(\x)$ by the surrogate function.
As the search space defines where optimal solutions are to be sought for, the search space definition is a means of infusing prior knowledge into the Bayesian Optimization. 
Usually, a search space is set so that the expected optimum is close to the center.

\subsection{High-dimensional Bayesian Optimization}


Even with its successes in many applications, several theoretical as well as practical issues~\cite{shahriari2016taking} still exist when employing Bayesian Optimization to real world problems.
Among others, many Bayesian optimization algorithms are restricted in practice to problems of moderate dimensions. 
In high dimensional problems, one suffers from the curse of dimensionality. 
To overcome the curse of dimensionality, several works make structural assumptions, such as low effective dimensionality~\cite{wang2016bayesian,bergstra2012random} or additive structure~\cite{kandasamy2015high}. 

Because of the way Gaussian Processes quantify uncertainty, the curse of dimensionality is a serious challenge for Gaussian Processes-based Bayesian Optimization in high dimensions. 
Since in high dimensions data points typically lie mostly on the boundary, and anyways far away from each other, the predictive variance tends to be higher in the regions near the boundary.
Thus, the acquisition function is somewhat biased to choose evaluations near the boundary, hence, biasing Bayesian Optimization towards solution near the boundary and away from the center, contradicting with the prior assumption.
This is the \textit{boundary issue}\cite{swersky2017improving}. 

\subsection{Contributions}

Different from the majority of the Bayesian Optimization methods that rely on a Euclidean geometry of the search space implicitly or explicitly\cite{hutter2011sequential,bergstra2011algorithms, snoek2012practical,snoek2014input,snoek2015scalable,swersky2013multi,wang2017batched}, the proposed BOCK applies a cylindrical geometric transformation on it.
The effect is that the volume near the center of the search space is expanded, while the volume near the boundary is shrunk.
Compared to \cite{snoek2014input}, where warping functions were introduced with many kernel parameters to be learned, we do not train transformations.
Also, we avoid learning many additional kernel parameters for better efficiency and scalability.
Because of the transformation, the proposed BOCK solves also the issue of flat optimization surfaces of the acquisition function in high dimensional spaces ~\cite{rana2017high}.
And compared to REMBO~\cite{wang2016bayesian}, BOCK does not rely on assumptions of low dimensionality of the latent search space.
\section{Method}
\subsection{Prior assumption and search space geometry}

The flexibility of a function $f$ on a high-dimensional domain $X$ can be, and usually is, enormous.
To control the flexibility and make the optimization feasible some reasonable assumptions are required.
A standard assumption in Bayesian Optimization is the \textit{prior assumption}~\cite{swersky2017improving}, according to which the optimum of $f(\x)$ should lie somewhere near the center of the search space $X$.
Since the search space is set with the \textit{prior assumption} in mind, it is reasonable for Bayesian Optimization to spend more evaluation budget in areas near the center of $X$. 

It is interesting to study the relation of the \textit{prior assumption} and the geometry of the search space.
The ratio of the volume of two concentric balls $B(\mathbf{c};R - \delta R)$ and $B(\mathbf{c};R)$, with a radius difference of $\delta R$, is
\begin{equation}
	\frac{\text{volume}(B(\mathbf{c};R - \delta R))}{\text{volume}(B(\mathbf{c};R))} = o((1-\delta)^D),
\end{equation}
which rapidly goes to zero with increasing dimensionality $D$.
This means that the volume of $B(\mathbf{c};R)$ is mostly concentrated near the boundary, 
which in combination with Gaussian processes' behavior of high predictive variance at points far from data, creates the \textit{boundary issue}~\cite{swersky2017improving}.
%

It follows, therefore, that with a transformation of the search space we could avoid excessively biasing our search towards large values of $R$. 

\subsection{Cylindrical transformation of search space}

The search space geometry has a direct influence on the kernel $K(\x, \x')$ of the Gaussian Process surrogate model, and, therefore, its predictive variance $\sigma^2(\x)$, see eq.~\eqref{eq:predictive_variance}.
A typical design choice for Gaussian Processes~\cite{snoek2012practical,snoek2014input,gonzalez2016batch} are stationary kernels, $K(\x, \x') \propto f(\x-\x')$.
Unfortunately, stationary kernels are not well equipped to tackle the \textit{boundary issue}.
Specifically, while stationary kernels compute similarities only in terms of relative locations $\x-\x'$, the boundary issue dictates the use of location-aware kernels $K(\x, \x')$ to recognize whether $\x, \x'$ lie near the boundary or the center areas of the search space.

A kernel that can address this should have the following two properties.
First, the kernel must define the similarity between two points $\x, \x'$ in terms of their absolute locations, namely the kernel has to be non-stationary.
Second, the kernel must transform the geometry of its input (\textit{i.e.}, the search space for the Gaussian Process surrogate model) such that regions near the center and the boundaries are equally represented.
To put it otherwise, we need a geometric transformation of the search space that expands the region near the center while contracting the regions near the boundary.
A transformation with these desirable properties is the cylindrical one, separating the radius and angular components of a point $\x$, namely
\begin{align} \label{eq:transformation}
& T(\mathbf{x}) = 
\begin{cases}
	(\Vert \mathbf{x} \Vert_2, \mathbf{x} / \Vert \mathbf{x} \Vert_2) & \text{for } \Vert \mathbf{x} \Vert_2 \neq 0 \\
	(0, \mathbf{a}_{arbitrary}) & \text{for } \Vert \mathbf{x} \Vert_2 = 0 
\end{cases} \\
& T^{-1}(r, \mathbf{a}) = r \mathbf{a} \nonumber
\end{align}
where $\mathbf{a}_{arbitrary}$ is an arbitrarily chosen vector with unit $\mathcal{\ell}_2$-norm
\footnote{Another possible geometric transformation could be from rectangular to spherical coordinates.
Unfortunately, the inverse transformation from spherical to rectangular coordinate entails multiplication of many trigonometric functions, causing numerical instabilities because of large products of small numbers.}.

After applying the geometric transformation we arrive at a new kernel $K_{cyl}(\x_1, \x_2)$, which we will refer to as the cylindrical kernel.
The geodesic similarity measure (kernel) of $K_{cyl}$ on the \textit{transformed cylinder}, $T(X)$, is defined as
\begin{align} \label{eq:kernel}
& K_{cyl}(\x_1, \x_2) \nonumber \\
& = \widetilde{K}(T(\x_1), T(\x_2)) \\
& = K_{r}(r_1, r_2) \cdot K_{a}(\mathbf{a}_1, \mathbf{a}_2) \nonumber,
\end{align}
where the final kernel decomposes into a 1-D radius kernel $K_r$  measuring the similarity of the radii of $r_1, r_2$ and a angle kernel $K_a$.

For the angle kernel $K_{a}(\mathbf{a}_1, \mathbf{a}_2)$, we opt for a continuous radial kernel on the (hyper-)sphere~\cite{jayasumana2014optimizing},
\begin{equation} \label{eq:kernel_angle}
K_{d}(\mathbf{a}_1, \mathbf{a}_2) = \sum_{p=0}^{P} c_p (\mathbf{a}_1^T \mathbf{a}_2)^p, ~~~~c_p\geq 0, ~\forall p
\end{equation}
with trainable kernel parameters of $c_0, \cdots, c_P$ and $P$ user-defined.
The advantages of a continuous radial kernel is two-fold.
First, with increasing $P$ a continuous radial kernel can approximate any continuous positive definite kernel on the sphere with arbitrary precision~\cite{jayasumana2014optimizing}.
Second, the cylindrical kernel has $P + 1$ parameters, which is independent of the dimensionality of $X$.
This means that while the continuous radial kernel retains enough flexibility, only few additional kernel parameters are introduced, which are independent of the dimensionality of the optimization problem and can, thus, easily scale to more than 50 dimensions.
This compares favorably to Bayesian optimization with ARD kernels that introduce at least $d$ kernel parameters for a $d$-dimensional search space.

Although the boundary issue is mitigated by the cylindrical transformation of the search space, the prior assumption (good solutions are expected near the center) can be promoted.
To this end, and to reinforce the near-center expansion of the cylindrical transformation, we consider input warping~\cite{snoek2014input} on the radius kernel $K_{r}(r_1, r_2)$.
Specifically, we use the cumulative distribution function of the Kumaraswamy distribution, $Kuma(r\vert \alpha, \beta) = 1 - (1 - r^{\alpha})^{\beta}$ (with $\alpha>0,\beta>0$),
\begin{align}\label{eq:radius_warping}
	&K_{r}(r_1, r_2)  \nonumber \\
    &= K_{base}(Kuma(r_1\vert \alpha, \beta), Kuma(r_1\vert \alpha, \beta)) \\
    &= K_{base}(1 - (1 - r_1^{\alpha})^{\beta}, 1 - (1 - r_2^{\alpha})^{\beta} \vert \nonumber \alpha, \beta)
\end{align}
where the non-negative $a, b$ are learned together with the kernel parameters.
$K_{base}$ is the base kernel for measuring the radius-based similarity.
Although any kernel is possible for $K_{base}$, in our implementations we opt for the  Matern52 kernel used in Spearmint~\cite{snoek2012practical}.
By making radius warping concave and non-decreasing, $K_r$ and, in turn, $K_{cyl}$ focus more on areas with small radii.

\begin{figure}[t!]
\vskip 0.2in
\begin{center}
\centerline{\includegraphics[width=0.9\columnwidth]{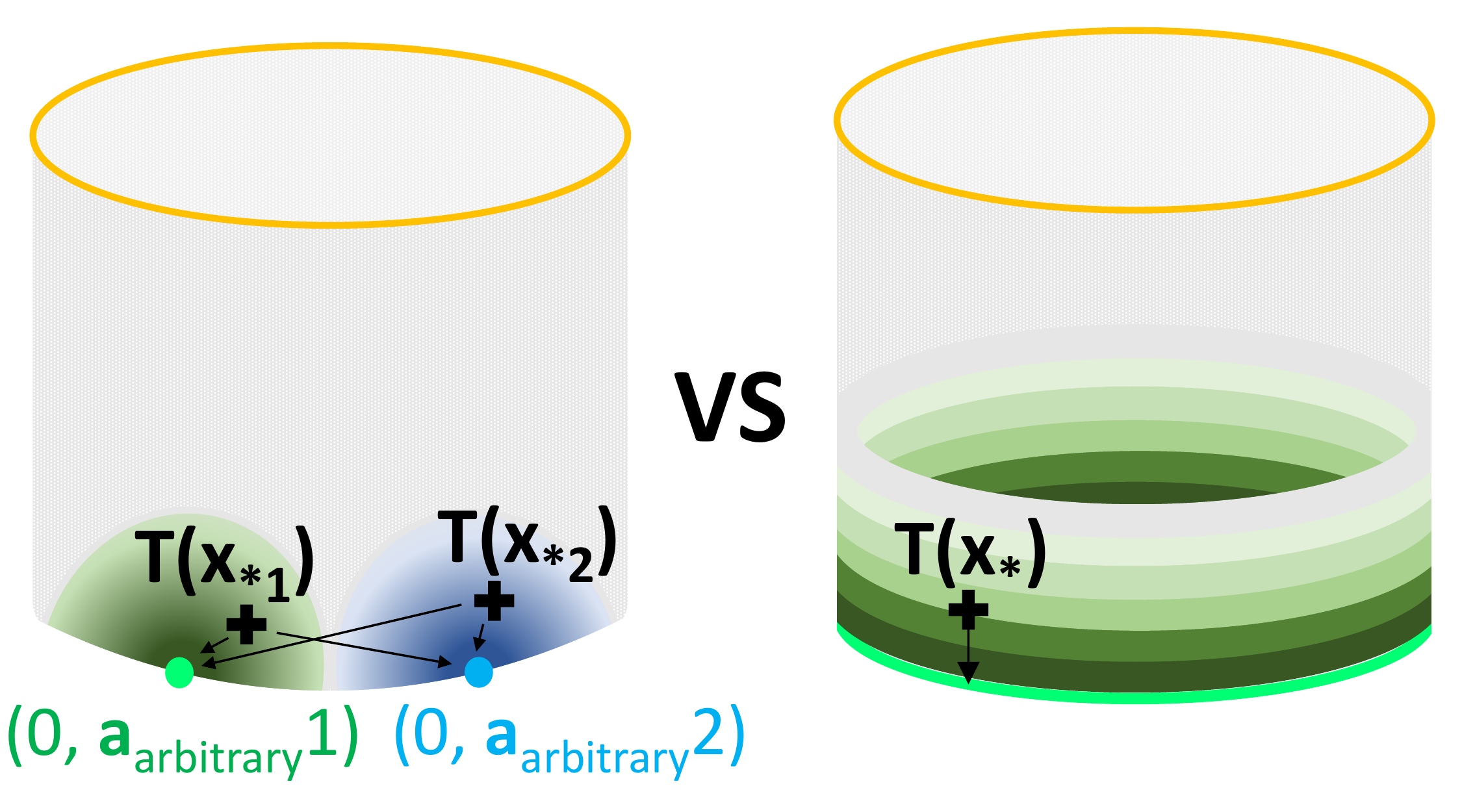}}
\caption{Similarity to the center point in \textit{transformed geometry}.}
\label{figure:special_treatment}
\end{center}
\vskip -0.2in
\end{figure}

Overall, the transformation of the search space has two effects.
The first effect is that the volume is redistributed, such that areas near the center are expanded, while areas near the boundaries are contracted.
Bayesian optimization's attention in the search space, therefore, is also redistributed from the boundaries to the center of the search space.
The second effect is that the kernel similarity changes, such that the predictive variance depends mostly on the angular difference between the existing data points and the ones to be evaluated.
An example is illustrated in Fig.~\ref{figure:transformation}, where our dataset comprises of $\D=\{\x_1, \x_2\}$ and the acquisition function must select between two points, $\x_{*, a}$ and $\x_{*, b}$.
Whereas in the original Euclidean geometry (Figure~\ref{figure:transformation} to the left) $\x_{*, a}$ is further away from $\D$, thus having higher predictive variance, in the cylindrical geometry both $\x_{*, a}$ and $\x_{*, b}$ are equally far, thus reducing the artificial preference to near-boundary points.

\subsection{Balancing center over-expansion}

The transformation $T$ maps an annulus $A(\mathbf{0};R-\delta R, R)$ of width $\delta R$ to the cylinder $C(R-\delta R, R; \mathbf{0}, 1)$, where $(\mathbf{0}, 1)$ is the center and the radius of the cylinder. 
For almost any point in the original ball there is a one-to-one mapping to a point on the cylinder.
The only exception is the extreme case of the ball origin, which is mapped to the 0-width sphere $C(0, 0; 0, 1) = \{(0, \mathbf{a}) \vert \Vert \mathbf{a} \Vert = 1 \}$ on the base of the cylinder (bright green circle in the Figure~\ref{figure:special_treatment} to the right). 
Namely, the center point $\x_{center}$ is overly expanded, corresponding to a set of points.
Because of the one-to-many correspondence between $\x_{center}$ and $C(0, 0; 0, 1)$, an arbitrary point is selected in eq.~\eqref{eq:transformation}. 

Unfortunately, the dependency on a point that is both arbitrary \textit{and} fixed incurs an arbitrary behavior of $K_{cyl}$ as well.
For any point $\x_* \in X \setminus \{\mathbf{0}\}$ the kernel $K_{cyl}(\x_{center}, \x_*)$ changes arbitrarily, depending on the choice of $\mathbf{a}_{arbitrary}$, see Figure~\ref{figure:special_treatment}.
Having a fixed arbitrary point, therefore, is undesirable as it favors points lying closer to it.
To this end, we define $\mathbf{a}_{arbitrary}$ as the angular component of a test point $\x_*$, $\mathbf{a}_{arbitrary} = \x_*/\Vert \x_* \Vert$, thus being not fixed anymore.
Geometrically, this is equivalent to using the point in $C(0, 0; 0, 1)$ closest to $T(\x_*)$, see Figure~\ref{figure:special_treatment} to the right.
This implies that, if the origin is in the dataset, the Gram matrix needed for computing the predictive density now depends on the angular location of the test point under consideration.
This is somewhat unconventional but still well behaved (the kernel is still positive definite and the predictive mean and variance change smoothly).
More details can be found in the supplementary material.



\section{Experiments}
\begin{table*}[t]
\centering
\caption{Bayesian Optimization on four benchmark functions for 20 and 200 dimensions, with the exception of Spearmint+~\cite{snoek2014input} and Elastic BO~\cite{rana2017high} evaluated only on the 20-dimensional cases because of prohibitive execution times).
For benchmark functions with complicated dependencies between variables (repeated Hartmann6, Rosenbrock), BOCK consistently discovers good solutions compared to other baselines, while not being affected by an increasing number of dimensions. 
Also, BOCK matches the accuracies of methods, like Matern, designed to exploit specific geometric structures, \eg, the additive structures of repeated Branin and Levy.
We conclude that BOCK is accurate, especially when we have no knowledge of the geometric landscape of the evaluated functions.
}
\label{tbl:benchmark}
\vskip 0.15in
\begin{center}
\begin{sc}
\begin{adjustbox}{max width=\textwidth}
\begin{threeparttable}
\begin{tabular}{lccccccccr}
\toprule
Benchmark & \multicolumn{2}{c}{Repeated Branin} & \multicolumn{2}{c}{Repeated Hartmann6} & \multicolumn{2}{c}{Rosenbrock} & \multicolumn{2}{c}{Levy} \\
\cmidrule(lr){2-3} \cmidrule(lr){4-5} \cmidrule(lr){6-7} \cmidrule(lr){8-9}  
Dimensions & 20 & 100 & 20 & 100 & 20 & 100 & 20 & 100 \\
\cmidrule(lr){2-3} \cmidrule(lr){4-5} \cmidrule(lr){6-7} \cmidrule(lr){8-9}
Minimum & 0.3979 & 0.3979 & -3.3223 & -3.3223 & 0.0000 & 0.0000 & 0.0000 & 0.0000 \\
\midrule 
SMAC               &  15.95$\pm$3.71&  20.03$\pm$0.85& -1.61$\pm$0.12& -1.16$\pm$0.19&   8579.13$\pm$~~~~  58.45&  8593.09$\pm$~~~~  18.80& 2.35$\pm$0.00& 9.60$\pm$0.04\\
TPE      	      & ~~7.59$\pm$1.20&  23.55$\pm$0.73& -1.74$\pm$0.10& -1.01$\pm$0.10&   8608.36$\pm$~~~~~~~0.00&  8608.36$\pm$~~~~~~~0.00& 2.35$\pm$0.00& 9.62$\pm$0.00\\
Spearmint       & ~~5.07$\pm$3.01& ~~2.78$\pm$1.06& -2.60$\pm$0.42& -2.55$\pm$0.19&   7970.05$\pm$    1276.62&  8608.36$\pm$~~~~~~ 0.00& 1.88$\pm$0.59& 4.87$\pm$0.35\\
Spearmint+& ~~6.83$\pm$0.32&           -& -2.91$\pm$0.25&              -&   5909.63$\pm$    2725.76&                        -& 2.35$\pm$0.00&             -\\
Additive BO\tnote{*}& ~~5.75$\pm$0.93&  14.07$\pm$0.84& -3.03$\pm$0.13& -1.69$\pm$0.22&   3632.25$\pm$    1642.71&  7378.27$\pm$~~   305.24& 2.32$\pm$0.02& 9.59$\pm$0.04\\
Elastic BO          & ~~6.77$\pm$4.85&               -& -2.85$\pm$0.57&              -&   5346.96$\pm$    2494.89&                        -& 1.35$\pm$0.34& - \\
Matern            & ~~0.41$\pm$0.00& ~~0.54$\pm$0.06& -3.29$\pm$0.04& -2.91$\pm$0.26&  ~~230.25$\pm$~~   187.41&  ~~231.42$\pm$~~~~ 28.94& 0.38$\pm$0.13& 2.17$\pm$0.18\\
\midrule
BOCK 	   & ~~0.50$\pm$0.12& ~~1.03$\pm$0.17& -3.30$\pm$0.02& -3.16$\pm$0.10& ~~~~47.87$\pm$~~~~ 33.94&  ~~128.69$\pm$~~~~ 52.84& 0.54$\pm$0.13& 6.78$\pm$2.16\\
\bottomrule
\end{tabular}
\begin{tablenotes}
\item[*] Additive BO~\cite{kandasamy2015high} requires a user-specified ``maximum group size'' to define the additive structure. In each experiment we tried 5 different values and reported the best result.
\end{tablenotes}
\end{threeparttable}
\end{adjustbox}
\end{sc}
\end{center}
\vskip -0.1in
\end{table*}

In Bayesian optimization experiments, we need to define \textit{(a)} how to train the surrogate model, \textit{(b)} how to optimize the acquisition function and \textit{(c)} how to set the search space.
For BOCK we use Gaussian Process surrogate models, where following~\cite{snoek2012practical,snoek2014input} we train parameters of BOCK with MCMC (slice sampling~\cite{murray2010slice,neal2003slice}) . 
For the acquisition function, we use the Adam~\cite{kingma2014adam} optimizer, instead of L-BFGS-B~\cite{zhu1997algorithm}. 
To begin the optimization we feed 20 initial points to Adam.
To select the 20 initial points, a sobol sequence~\cite{bratley1988algorithm} of 20,000 points is generated on the cube (we used the cube for fair comparison with others).
The acquisition function is evaluated on these points and the largest 20 points are chosen as the initial ones.
Instead of using a static sobol sequence in the entire course of Bayesian optimization ~\cite{snoek2012practical,snoek2014input}, we generate different sobol sequences for different evaluations, as fixed grid point impose too strong constraints in high dimensional problems.
In the $d$-dimensional space, our search space is a ball $B(\mathbf{0}, \sqrt{d})$ circumscribing a cube $[-1, 1]^{d}$, which is the scaled and translated version of the typical search region, unit cube $[0, 1]^{d}$.
Our search space is much larger than a cube.
By generating sobol sequence on the cube, the reduction of the \textit{boundary issue} mostly happens at corners of the cube $[-1, 1]^{d}$.

\subsection{Benchmarking}

First, we compare different Bayesian Optimization methods and BOCK on four benchmark functions.
Specifically, following~\cite{eggensperger2013towards,laguna2005experimental} we use the repeated Branin, repeated Hartmann6 and Levy to assess Bayesian Optimization in high dimensions.
To test the ability of Bayesian Optimization methods to optimize functions with more complex structure and stronger intra-class dependencies, we additionally include the Rosenbrock benchmark, typically used as benchmark for gradient-based optimization~\cite{laguna2005experimental}.
The precise formulas for the four benchmark functions are added to the supplementary material.
%
%
We solve the benchmark functions in 20 and 100 dimensions
\footnote{We also solve the 50-dimensional cases. As conclusions are similar, we add these results to the supplementary material.
}
, using 200 and 600 function evaluations respectively for all Bayesian Optimization methods.
We compare the proposed BOCK with the following Bayesian Optimization methods using publicly available software:  SMAC~\cite{hutter2011sequential}, TPE~\cite{bergstra2011algorithms}, Spearmint~\cite{snoek2012practical}, Spearmint+~\cite{snoek2014input}, additive BO~\cite{kandasamy2015high}, elastic BO~\cite{rana2017high}.
We also report an in-house improved Spearmint implementation, which we refer to as Matern.
\footnote{Differences with standard Spearmint: (a) a non-ARD, Matern52 kernel for the surrogate model, (b) dynamic search grid generation per evaluation, (c) Adam~\cite{kingma2014adam} instead of L-BFGS-B~\cite{zhu1997algorithm}, (d) more steps for optimizer.}

We focus on four aspects: \textit{(a)} accuracy, \textit{(b)} efficiency (wall clock time) vs accuracy, \textit{(c)} scalability (number of dimensions) vs efficiency, and \textit{(d)} robustness of BOCK to hyperpararameters and other design choices.
We study (a) in all four benchmark functions.
For brevity, we report (b)-(d) on the Rosenbrock benchmark only, the hardest of the four benchmark functions for all Bayesian Optimization methods in terms of accuracy, and report results the rest of the benchmark functions in the supplementary material.

\begin{figure}[!ht]
	\centering
	\begin{adjustbox}{max width=0.7\linewidth}
  	\includegraphics[width=\linewidth]{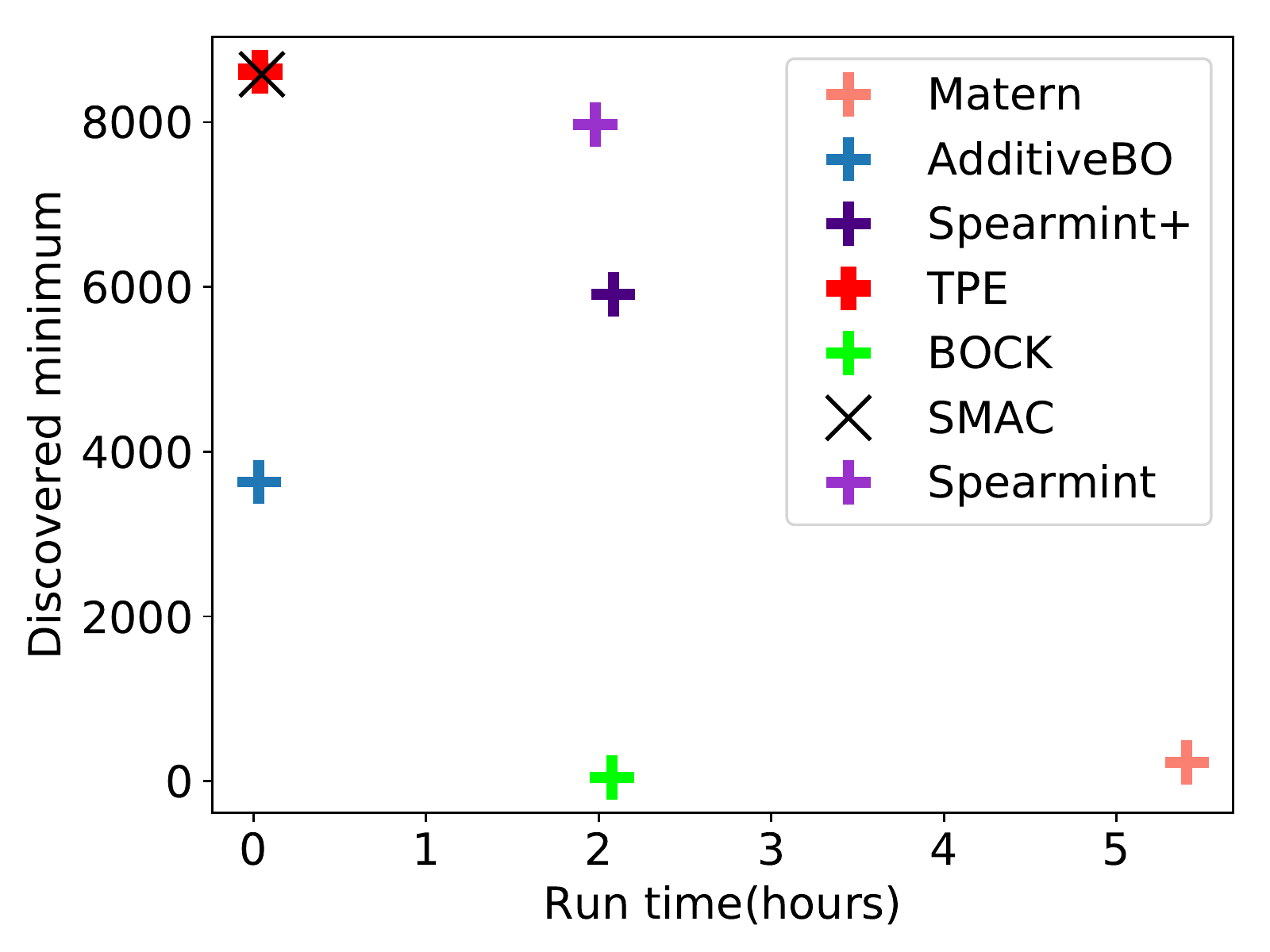}
    \end{adjustbox}
\caption{Accuracy vs wall clock time efficiency for the 20-dimensional Rosenbrock benchmark.
BOCK is the closest to the optimum operating point $(0, 0)$.
Matern is also accurate enough, although considerably slower, while SMAC and additive BO are faster but considerably less accurate.
}
\label{figure:efficiency}
\end{figure}

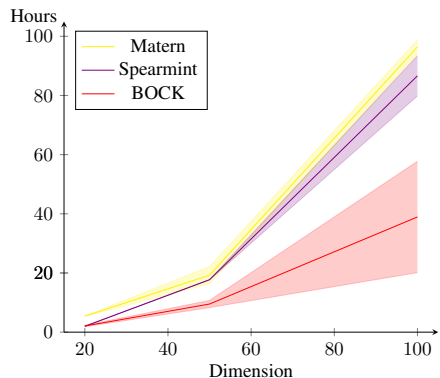
\begin{figure}[!ht]
	\centering
\begin{adjustbox}{max width=0.7\linewidth}
	\begin{tikzpicture}[scale=1.0]
\begin{axis}[
    axis x line=bottom,
	axis y line=left,
    xlabel={Dimension}, 
    xlabel style={at={(0.5, -0.08)}}, 
    ylabel={Hours}, 
	ylabel style={at={(0.0, 1.02)}, rotate=-90}, 
    xmin=15,
	xmax=105,
	ymin=0,
	ymax=105,
	extra y ticks={20},
    legend pos=north west]

\addplot[color=\colormatern, mark=none, name path=cube_mean] table [x=dim, y=mean_hour] {run_time_data/cube.txt};
\addplot[color=\colormatern, mark=none, name path=cube_lower, opacity=0.2, forget plot] table [x=dim, y expr=\thisrow{mean_hour} - \thisrow{std_hour}] {run_time_data/cube.txt};
\addplot[color=\colormatern, mark=none, name path=cube_upper, opacity=0.2, forget plot] table [x=dim, y expr=\thisrow{mean_hour} + \thisrow{std_hour}] {run_time_data/cube.txt};
\addplot[color=\colormatern, fill=\colormatern, fill opacity=0.2, forget plot] fill between[of=cube_lower and cube_upper];
\addlegendentry{Matern}

\addplot[color=\colorspearmint, mark=none, name path=spearmint_mean] table [x=dim, y=mean_hour] {run_time_data/spearmint.txt};
\addplot[color=\colorspearmint, mark=none, name path=spearmint_lower, opacity=0.2, forget plot] table [x=dim, y expr=\thisrow{mean_hour} - \thisrow{std_hour}] {run_time_data/spearmint.txt};
\addplot[color=\colorspearmint, mark=none, name path=spearmint_upper, opacity=0.2, forget plot] table [x=dim, y expr=\thisrow{mean_hour} + \thisrow{std_hour}] {run_time_data/spearmint.txt};
\addplot[color=\colorspearmint, fill=\colorspearmint, fill opacity=0.2, forget plot] fill between[of=spearmint_lower and spearmint_upper];
\addlegendentry{Spearmint}

\addplot[color=\colorspherewarpingorigin, mark=none, name path=spherewarpingorigin_mean] table [x=dim, y=mean_hour] {run_time_data/spherewarpingorigin.txt};
\addplot[color=\colorspherewarpingorigin, mark=none, name path=spherewarpingorigin_lower, opacity=0.2, forget plot] table [x=dim, y expr=\thisrow{mean_hour} - \thisrow{std_hour}] {run_time_data/spherewarpingorigin.txt};
\addplot[color=\colorspherewarpingorigin, mark=none, name path=spherewarpingorigin_upper, opacity=0.2, forget plot] table [x=dim, y expr=\thisrow{mean_hour} + \thisrow{std_hour}] {run_time_data/spherewarpingorigin.txt};
\addplot[color=\colorspherewarpingorigin, fill=\colorspherewarpingorigin, fill opacity=0.2, forget plot] fill between[of=spherewarpingorigin_lower and spherewarpingorigin_upper];
\addlegendentry{BOCK}

\end{axis}
\end{tikzpicture}
\end{adjustbox}
\caption{Wall clock time(hours) on the Rosenbrock benchmark for an increasing the number of dimensions (20, 50 and 100 dimensions, using 200, 400 and 600 function evaluations respectively for all methods).
The solid lines and colored regions represent the mean wall clock time and one standard deviation over these 5 runs.
As obtaining the evaluation score $y=f(\x_*)$ on these benchmark functions is instantaneous, the wall clock time is directly related to the computational efficiency of algorithms.
In this figure, we compare BOCK and BOs with relative high accuracy in all benchmark functions, such as Spearmint and Matern.
BOCK is clearly more efficient, all the while being less affected by the increasing number of dimensions.
}
\label{figure:scalability}
\end{figure}

\noindent\textbf{Accuracy.}
We first present the results regarding the accuracy of BOCK and the Bayesian Optimization baselines in Table~\ref{tbl:benchmark}.
BOCK and Matern outperform others with large margin in discovering near optimal solutions.
For benchmark functions with complicated dependencies between variables, such as the repeated Hartmann6 and Rosenbrock, BOCK consistently discovers smaller values compared to other baselines, while not being affected by an increasing number of dimensions. 
What is more, BOCK is on par even with methods that are designed to exploit the specific geometric structures, if the same geometric structures can be found in the the evaluated functions.
For instance, the repeated Branin and Levy have an additive structure, where the same low dimensional structure is repeated.
The non-ARD kernel of Matern can exploit such special, additive structures.
BOCK is able to reach a similar near-optimum solution without being explicitly designed to exploit such structures.

We conclude that BOCK is accurate, especially when we have no knowledge of the geometric landscape of the evaluated functions.
In the remaining of the experiments we focus on the Bayesian Optimization methods with competitive performance, namely BOCK, Spearmint and Matern.

\noindent\textbf{Efficiency vs accuracy.}
Next, we compare in Figure~\ref{figure:efficiency} the accuracy of the different Bayesian Optimization methods as a function of their wall clock times for the 20-dimensional case for Rosenbrock.
As the function minimum is $f(\x_{opt})=0$, the optimal operating point is at $(0, 0)$.
BOCK is the closest to the optimal point.
Matern is the second most accurate, while being considerably slower to run.
SMAC~\cite{hutter2011sequential} and AdditiveBO~\cite{kandasamy2015high} are faster than BOCK, however, they are also considerably less accurate.

\noindent\textbf{Scalability.}
In Figure~\ref{figure:scalability} we evaluate the most accurate Bayesian Optimization methods from Table~\ref{tbl:benchmark} (Spearmint, Matern and BOCK.) with respect to how scalable they  are, namely measuring the wall clock time for an increasing number of dimensions.
Compared to Spearmint BOCK is less affected by the increasing number of dimensions. 
Not only the BOCK surrogate kernel requires fewer parameters, but also the number of surrogate kernel parameters is independent of the number of input dimensions, thus making the surrogate model fitting faster.
BOCK is also faster than Matern, although the latter uses a non-ARD kernel that is also independent of the number of input dimensions.
Presumably, this is due to a better, or smoother, optimization landscape after the cylindrical transformation of geometry of the input space, affecting positively the search dynamics.
We conclude that BOCK is less affected by the increasing number of dimensions, thus scaling better.

\noindent\textbf{Robustness.}
To study the robustness of BOCK to design choices, we compare three BOCK variants.
The first is the standard BOCK as described in Section 3.
The second variant, BOCK-W, removes the input warping on the radius component.
The third variant, BOCK+B, includes an additional boundary treatment to study whether further reduction of the predictive variance is beneficial.
Specifically, we reduce the predictive variance by adding ``fake'' data. 
%
\footnote{Predictive variance depends only on the inputs $\x$, not the evaluations $y=f(\x)$. 
Thus we can manipulate the predictive variance only with input data. 
BOCK+B uses one additional ``fake data'', which does not have output value(evaluation), in its predictive variance. 
BOCK's predictive variance $\sigma^2(\x_* \vert \D)$ becomes $\sigma^2(\x_* \vert \D \cup \{(R \x_*/\Vert \x_* \Vert, \sim)\})$ in BOCK+B on the search space of the ball $B(\mathbf{0};R)$, where $(R \x_*/\Vert \x_* \Vert , \sim)$ is the fake data.
}
We present results in Table~\ref{tbl:bock_variants}.

\begin{table}[t]
\caption{Comparison between different BOCK variants on Rosenbrock.
Excluding input warping results in slight instabilities, while including additional boundary treatments brings only marginal benefits.}
\label{tbl:bock_variants}
\vskip 0.1in
\begin{center}
\begin{sc}
\begin{adjustbox}{max width=\linewidth}
\begin{tabular}{lccc}
\toprule
Dimensions & 20 & 50 & 100 \\
\midrule
BOCK 	& ~~~~47.87$\pm$~~~~ 33.94 & 29.65$\pm$11.56 &  128.69$\pm$~~52.84 \\
BOCK-W  &   1314.03$\pm$   1619.73 & 51.14$\pm$58.18 &  157.89$\pm$ 161.92 \\
BOCK+B  & ~~~~48.87$\pm$~~~~ 18.33 & 33.90$\pm$21.69 & ~~87.00$\pm$~~36.88 \\
\bottomrule
\end{tabular}
\end{adjustbox}
\end{sc}
\end{center}
\vskip -0.1in
\end{table}

\begin{figure*}[!t]
\minipage{0.32\textwidth}
	\centering
    \textbf{100 dim, $\mathbf{W}_2$ : $10 \times 10$}\par
  	\includegraphics[width=\linewidth]{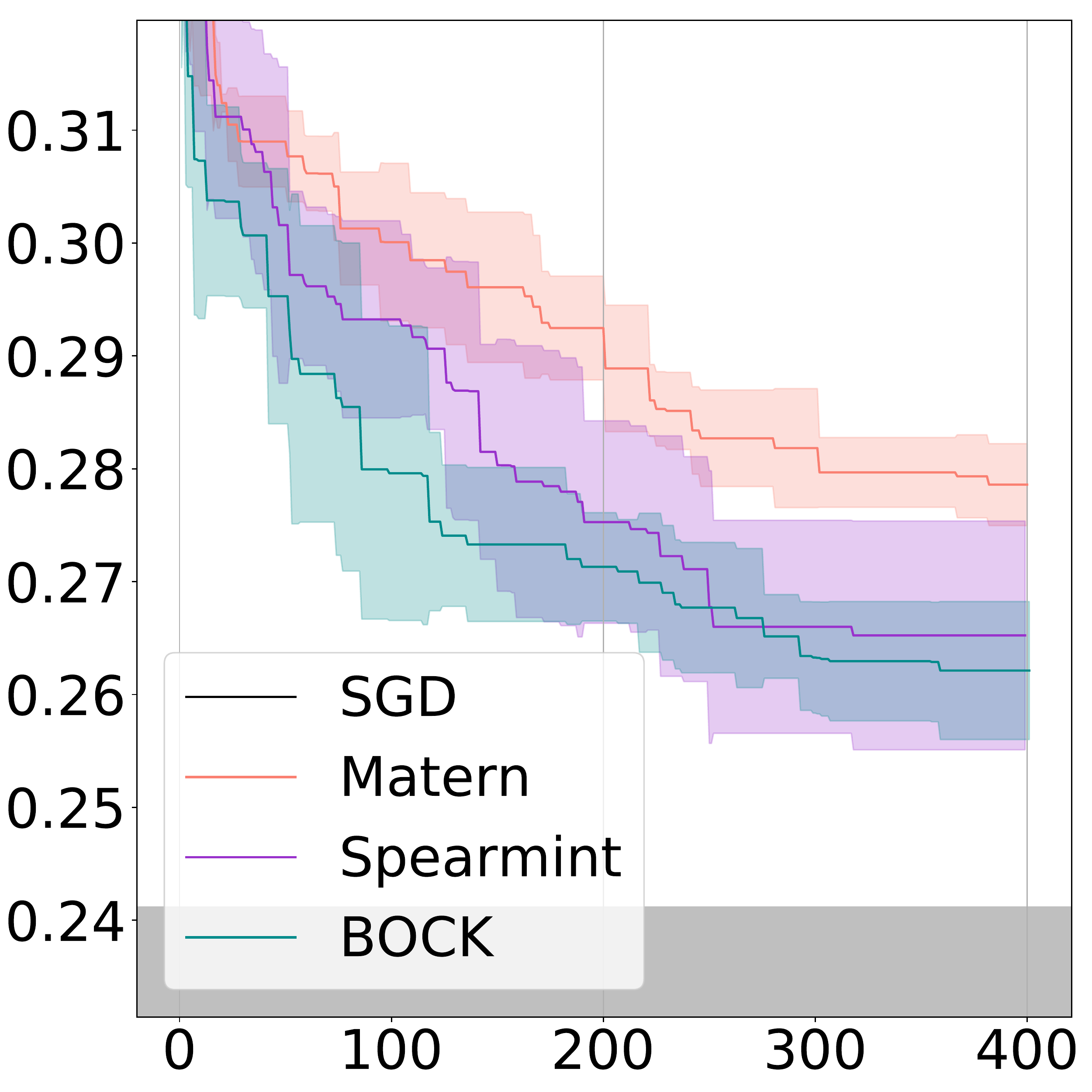}
\endminipage\hfill
\minipage{0.32\textwidth}
\centering
    \textbf{200 dim, $\mathbf{W}_2$ : $20 \times 10$}\par
  	\includegraphics[width=\linewidth]{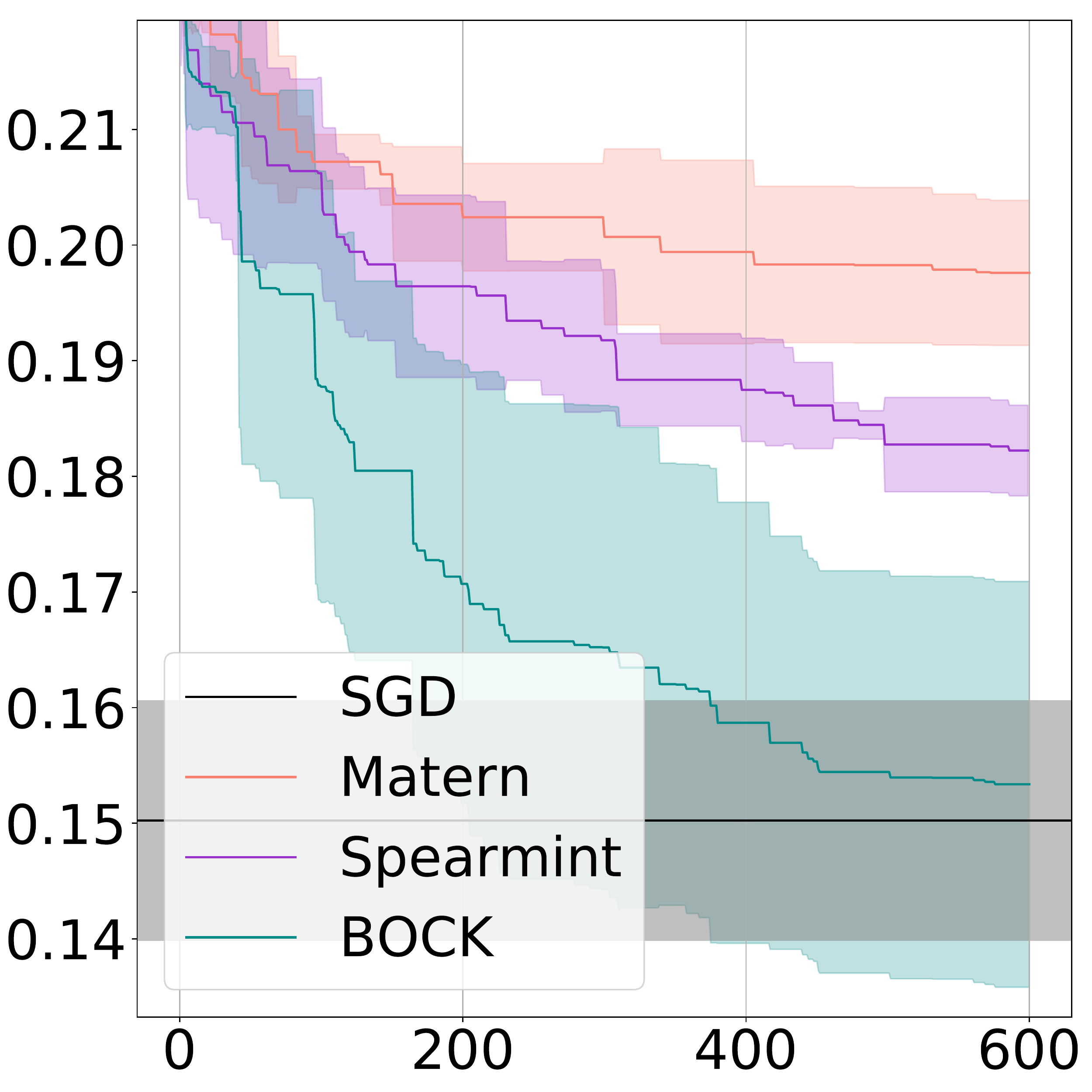}
\endminipage\hfill
\minipage{0.32\textwidth}%
\centering
    \textbf{500 dim, $\mathbf{W}_2$ : $50 \times 10$}\par
  	\includegraphics[width=\linewidth]{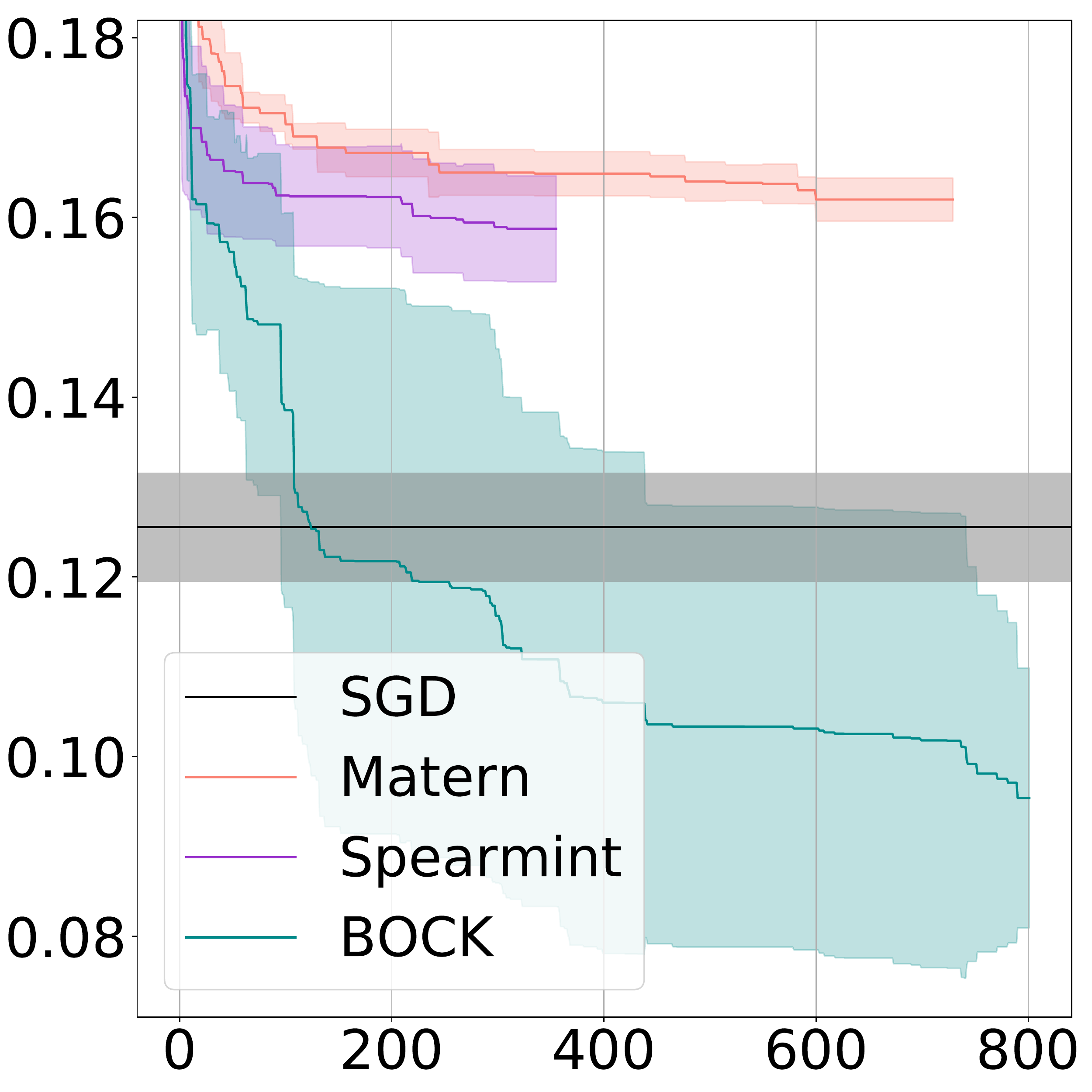}
\endminipage
\vspace{-0.1cm}
\caption{
Training on MNIST a two-layer neural network: $784 \xrightarrow{\mathbf{W}_1, \mathbf{b}_1} N_{hidden} \xrightarrow{\mathbf{W}_2, \mathbf{b}_2} 10$.
For all experiments, $\mathbf{W}_1$, $\mathbf{b}_1$ and $\mathbf{b}_2$ are optimized with Adam~\cite{kingma2014adam} and $\mathbf{W}_2$ with Bayesian Optimization.
In this experiment, Bayesian Optimization repeats the following steps.
\textit{(a)} A new $\mathbf{W}_2$ is suggested by BOCK.
\textit{(b)} Given this $\mathbf{W}_2$, the $\mathbf{W}_1$, $\mathbf{b}_1$, $\mathbf{b}_2$ are fine-tuned by SGD on the training set.
\textit{(c)} The loss on the test is returned as the evaluation on $\mathbf{W}_2$.
Therefore, in this experiment, Bayesian Optimization algorithms directly optimize the test loss.
As an additional baseline we train the neural network only with SGD on the training set and report the loss on the test set.
We report the mean loss on the test set (solid line), as well and $\pm 1 \cdot$ standard deviation over 5 runs (colored area).
We observe that BOCK can optimize successfully a modestly sized neural network layer. BOCK consistently finds a better solution than existing Bayesian optimization algorithms. In high dimensional cases, BOCK outperforms other algorithms with a significant margin.
We conclude that BOCK is capable of optimizing in high-dimensional and complex spaces. 
}
\label{figure:nn_experiment}
\end{figure*}

Removing the input warping on the radius is hurting the robustness, as BOCK-W tends to reach slightly worse minima than BOCK.
However, introducing further boundary treatments has a marginal effect.

Further, we assess the sensitivity of BOCK with respect to the hyperparameter P in eq.\eqref{eq:transformation}.
For $P=3, 5, 7, 9$, we observe that higher P tends to give slightly better minima, while increasing the computational cost.


For clarity of presentation, as well as to maintain the experimental efficiency, in the rest of the experiments we focus on BOCK with $P=3$.

%

\subsection{Optimizing a neural network layer}
\label{sec:mnist}

As BOCK allows for accurate and efficient Bayesian Optimization for high-dimensional problems, we next perform a stress test, attempting to optimize neural network layers of 100, 200 and 500 dimensions.
Specifically, we define a two-layered neural network with architecture: $784 \xrightarrow{\mathbf{W}_1, \mathbf{b}_1} N_{hidden} \xrightarrow{\mathbf{W}_2, \mathbf{b}_2} 10$, using ReLU as the intermediate non-linearity.

In this experiment we are \textit{only} interested in the optimization ability of BOCK of the parameters of a neural network, not in its ability to find solutions that generalize well. Thus, we intentionally follow a procedure that tests if BOCK is able to even overfit to the test set.
Specifically, for all Bayesian optimization experiments $\mathbf{W}_1$, $\mathbf{b}_1$ and $\mathbf{b}_2$ are optimized with Adam~\cite{kingma2014adam} and $\mathbf{W}_2$ with Bayesian Optimization.
The training proceeds as follows.
First, Bayesian Optimization suggests a $\mathbf{W}_2$ based on evaluations on the \textbf{test} set.
Given this $\mathbf{W}_2$ we train on the \textbf{train+validation} sets the $\mathbf{W}_1, b_1, b_2$ with Adam, then repeat.
We show results in Figure~\ref{figure:nn_experiment}, where we report mean and standard deviation over 5 runs for all methods.
We compare BOCK with the competitive Spearmint and Matern.
As baseline, we train a network with Adam~\cite{kingma2014adam} on the training set and report the test loss.
To the best of our knowledge we are the first to apply Gaussian Process-based Bayesian Optimization in so high-dimensional and complex, representation learning spaces.
\footnote{To our knowledge, running Bayesian Optimization on 200 or 500 dimensional problems has only been tried with methods assuming low effective dimensionality~\cite{wang2016bayesian,chen2012joint}}.

We observe that BOCK clearly outperforms Spearmint and Matern, with the gap increasing for higher $\mathbf{W}_2$ dimensions.
What is more surprising, however, is that BOCK is able to match and even outperform the Adam-based SGD in the 200 and 500-dimensional experiments for all 5 runs.
There are two reasons for this.
First, in this experiment, all Bayesian optimization algorithms directly optimize the test loss.
Second, in its sophistication Adam~\cite{kingma2014adam} probably overfits to the training set.

In the end, the final neural network is obviously \textbf{not optimal} in terms of generalization, as to optimize $\mathbf{W}_2$ BOCK has access to the test set.
However, even the fact that it is possible to optimize such high-dimensional and complex (representation learning) functions with Bayesian Optimization is noteworthy.
We conclude that BOCK is able to optimize complex, multiple-optima functions, such as neural network layers.

\subsection{Hyper-optimizing stochastic depth ResNets}

As BOCK allows for accurate and efficient Bayesian Optimization, in our last experiment we turn our attention to a practical hyperparameter optimization application.
Stochastic Depth ResNet (SDResNet)~\cite{huang2016deep} was shown to obtain better accuracy and faster training by introducing a stochastic mechanism that randomly suppresses ResNet blocks (ResBlock)~\cite{he2016deep}.
The stochastic mechanism for dropping ResBlocks is controlled by a vector $p\in[0, 1]^t$ of probabilities for $t$ ResBlocks, called ``death rate''.
In~\cite{huang2016deep} a linearly increasing (from input to output) death rate was shown to improve accuracies.

Instead of pre-defined death rates, we employ BOCK to find the optimal death date vector for SDRes-110 on CIFAR100~\cite{krizhevsky2009learning}.
We first train an SD-ResNet for 250 epochs and linear death rates with exactly the same configuration in~\cite{huang2016deep} up to 250 epochs.
In this experiment BOCK has access to the \textbf{training and validation set only}.
Then, per iteration BOCK first proposes the next candidate $p$ based on evaluation on the validation set.
Given the candidate $p$ we run 100 epochs of SGD on the training set and repeat with an annealed learning rate (0.01 for 50 epochs, then 0.001 for 50 more).
We initialize the death rate vector to $p=[0.5, 0.5, ..., 0.5]$.
We report the final accuracies computed in the \textbf{unseen test set} in Table~\ref{tbl:stochastic-resnet}, using only 50 evaluations.

We observe that BOCK learns a $p$-value that results in an improved validation accuracy compared to SDResNet, all the while allowing for a lower expected depth.
The improved validation accuracy materializes to an only slightly better test accuracy, however.
One reason is that optimization is not directly equivalent to learning, as also explained in Section~\ref{sec:mnist}.
What is more, it is likely that the accuracy of SDResNet-110 on CIFAR-100 is maxed out, especially considering that only 50 evaluations were made.
We conclude that BOCK allows for successful and efficient Bayesian Optimization even for practical, large-scale learning problems.

\begin{table}[t]
\vspace{-0.5cm}
\caption{Using BOCK to optimize the ``death rates'' of a Stochastic Depth ResNet-110, we improve slightly the accuracy on CIFAR100 while reducing the expected depth of the network.}
\vspace{-0.4cm}
\label{tbl:stochastic-resnet}
\vskip 0.15in
\begin{center}
\begin{sc}
\begin{adjustbox}{max width=\linewidth}
\begin{tabular}{lcccr}
\toprule
Method & Test Acc. & Val. Acc. & Exp. Depth \\
\midrule
ResNet-110          & 72.98$\pm0.43$& 73.03$\pm$0.36& 110.00\\
SDResNet-110+Linear & 74.90$\pm$0.15& 75.06$\pm$0.04& ~~82.50\\
SDResNet-110+BOCK   & 75.06$\pm$0.19& 75.21$\pm$0.05& 74.51$\pm$1.22\\
\bottomrule
\end{tabular}
\end{adjustbox}
\end{sc}
\end{center}
\vskip -0.1in
\vspace{-0.6cm}
\end{table}

\section{Conclusion}
We propose BOCK, Bayesian Optimization with Cylindrical Kernels.
Many of the problems in Bayesian Optimization relate to the \textit{boundary issue} (too much value near the boundary), and the \textit{prior assumption} (optimal solution probably near the center).
Because of the boundary issue, not only much of the evaluation budget is unevenly spent to the boundaries, but also the \textit{prior assumption} is violated.
The basic idea behind BOCK is to transform the ball geometry of the search space with a cylindrical transformation, expanding the volume near the center while contracting it near the boundaries.
As such, the Bayesian optimization focuses less on the boundaries and more on the center.

We test BOCK extensively in various settings.
On standard benchmark functions BOCK is not only more accurate, but also more efficient and scalable compared to state-of-the-art Bayesian Optimization alternatives.
Surprisingly, optimizing a neural network (on the test set) up to 500 dimensions with BOCK allows for even better (albeit overfitting) parameters than SGD with Adam~\cite{kingma2014adam}.
And hyper-optimizing the ``death rate'' of stochastic depth ResNet~\cite{huang2016deep} results in smaller ResNets while maintaining accuracy.

We conclude that BOCK allows for accurate, efficient and scalable Gaussian Process-based Bayesian Optimization.  We plan to make the code public upon acceptance.


\bibliography{bayesian_optimization}
\bibliographystyle{icml2018}

\newpage
\onecolumn
\setcounter{page}{1}
\setcounter{section}{0}
\icmltitlerunning{BOCK : Bayesian Optimization with Cylindrical Kernels - Supplementary}
\icmltitle{BOCK : Bayesian Optimization with Cylindrical Kernels\\Supplementary Materials}





\icmlaffiliation{equal}{QUvA Lab, Informatic Institute, University of Amsterdam, Amsterdam, Netherlands}
\icmlaffiliation{cifar}{Canadian Institute for Advanced Research, Toronto, Canada}

\icmlcorrespondingauthor{ChangYong Oh}{changyong.oh0224@gmail.com}
\icmlcorrespondingauthor{Efstratios Gavves}{efstratios.gavves@gmail.com}
\icmlcorrespondingauthor{Max Welling}{m.welling@uva.nl}

\icmlkeywords{Bayesian Optimization, Hyperparameter Optimization, Gaussian Process, Kernel, Geometry, Warping, Radial kernel, High dimensional spaces, Change of coordinates, Transformation, Resolution, Boundary issue}

\vskip 0.3in




\section{Special Treatment of the center point}

In Section3.3, we propose the special treatment on the center point to correct the problem resulting from over-expansion of the center point.
We provide a justification for the positive semi-definiteness of $K_{cyl}$.

Since the cylindrical kernel $K_{cyl}$ is a tensor product of the kernel $K_r$ from the radius component, and the kernel $K_a$ from the angular component, if we can show that both $K_r$ and $K_d$ are proper kernels (\ie, positive semi-definite), then we can conclude that $K_{cyl}$ is also a proper kernel~\cite{rasmussen2006gaussian}. 

Let us denote with $T : B(\mathbf{0},R) \rightarrow C(0, R;\mathbf{0}, 1)$ the transformation from a ball to a cylinder, and with $\pi_a$ the projection to angle component in a cylinder.
For a given set $\widetilde{\D} = \D \cup \{\mathbf{0}\}$, we denote the angle component $\pi_a(T(\D))$ as $\D_a$. Then the gram matrix of $K_a$ on $\widetilde{\D}$ can be represented by
\begin{equation}
	\begin{bmatrix}
		K_a(\D_a,\D_a) & K_a(\D_a,\aaa_{arbitrary}) \\
		K_a(\aaa_{arbitrary}, \D_a) & K_a(\aaa_{arbitrary}, \aaa_{arbitrary})
	\end{bmatrix}
\end{equation}
In the special treatment, we set $\aaa_{arbitrary} = \aaa* = \x_* / \Vert \x_* \Vert$.
This is nothing but the gram matrix of $K_a$ on the dataset ${\aaa_1, \aaa_2, \cdots, \aaa_N, \aaa_*}$
As long as, $K_a$ is proper kernel, using the special treatment does break the positive semi-definiteness of kernel.

The special treatment assumes $\x_* \ne \mathbf{0}$. 
A single point is of measure zero under any non-atomic measure. 
This assumption can be safely made, theoretically. 
In our experiments, we start with data including $\mathbf{0}$ as an initial data point, thus the acquisition function does not need to go over $\x_*=0$ anymore.

Interestingly, this special treatment bears similarity to Bayesian Optimization using treed Gaussian Processes~\cite{assael2014heteroscedastic}. 
When there is $\mathbf{0}$ in our training data set, at each prediction, we have a  Gaussian Process on the same data set but one point.
Namely, one can view this as having different Gaussian Processes at different prediction points, in the sense that the data conditioning the Gaussian Process change (not the kernel parameters).
As the treed Bayesian Optimization guarantees continuity between the regions having the different Gaussian Processes is also, the cylindrical kernel with the special treatment also has continuity since the Gram matrix is a continuous function of $\dd_arbitrary$.

However, at different prediction points we have different gram matrices.
Hence, a naive implementation of the above idea makes the maximization of the acquisition function infeasible.
In Gaussian process prediction, main computation bottle is to calculate a quadratic form as below
\begin{equation} \label{eq:quadratic}
	\begin{bmatrix}
    	\mathbf{p}^T & p_0
    \end{bmatrix}
    \Bigg(
    \begin{bmatrix}
		K_{cyl}(\D,\D) & K_{cyl}(\D,0) \\
		K_{cyl}(0, \D) & K_{cyl}(0, 0)
	\end{bmatrix}
    + \sigma^2_{obs} I \Bigg)^{-1}
    \begin{bmatrix}
    	\mathbf{q}^T \\
        q_0
    \end{bmatrix}
\end{equation}
Fortunately, we can calculate the quadratic form eq~\eqref{eq:quadratic} efficiently by using block matrix inversion. Once we calculate $K_{cyl}(\D,\D)^{-1}$, by using pre-calculated, $K_{cyl}(\D,\D)^{-1}$, calculating eq~\eqref{eq:quadratic} for different $\x_*$ requires marginal computation.

\subsection{Positive semi-definiteness of cylindrical kernels}
\begin{theorem}
If $K_a(\mathbf{a}, \mathbf{a}) = \eta > 0$, $\forall \mathbf{a} \in \mathcal{S}^{d-1}$, then cylindrical kernels are positive semi-definite with the special treatment of the centre point.
\end{theorem}
\begin{proof}
We need to show that $\forall n \in \mathbf{N}$, $\forall \mathbf{a}_i \in \mathbf{R}^n$, $\mathbf{c} \in \mathbf{R}^n$
\begin{equation}
    \mathbf{c}^T K_a(\D_a, \D_a) \mathbf{c} \ge 0
\end{equation}
where $\D_a = \pi_a(T(\D))$ and $\D = \{\mathbf{a}_i\}_{i=1,\cdots,n}$

Case 1. $\mathbf{0} \notin \D$.
In this case, this is positive definite as shown in~\cite{jayasumana2014optimizing}.

Case 2. $\mathbf{0} \in \D$.
Let us denote $\D = \D_{>0} \cup \{ \mathbf{0}_a \}$, where $\mathbf{0}_a = \pi_a(T(\mathbf{0}))$ and $[\mathbf{c}^T \quad c]^T \in \mathbf{R}^n$.
Then we need to show following
\begin{equation}\label{eq:pos_def_quad}
    \begin{bmatrix}
        \mathbf{c}^T & c
    \end{bmatrix}
    \begin{bmatrix}
        K(\D_{>0,a}, \D_{>0,a}) & \eta \textbf{1} \\
        \eta \textbf{1}^T & \eta
    \end{bmatrix}
    \begin{bmatrix}
        \mathbf{c} \\
        c
    \end{bmatrix}
    = \mathbf{c}^T K_{>0, a} \mathbf{c} - 2 \eta c \textbf{1}^T \mathbf{c} + \eta c^2 \ge 0
\end{equation}
where $\D_{>0,a} = \pi_a(\D_{>0})$, $\textbf{1} \in \mathbf{R}^{n-1}$ and $K_{>0, a} = K(\D_{>0,a}, \D_{>0,a})$.

After differentiation, we have
\begin{align}
    \frac{\partial}{\partial \mathbf{c}} \mathbf{c}^T K_{>0, a} \mathbf{c} - 2 \eta c \textbf{1}^T \mathbf{c} + \eta c^2 &= 2 K_{>0, a} \mathbf{c} - 2 \eta c \textbf{1} \\
    \frac{\partial}{\partial c} \mathbf{c}^T K_{>0, a} \mathbf{c} - 2 \eta c \textbf{1}^T \mathbf{c} + \eta c^2 &= -2 \eta \textbf{1}^T \mathbf{c} + 2 \eta c
\end{align}
Thus the minimum of the quadratic form eq.\ref{eq:pos_def_quad} satisfies
\begin{align}
    K_{>0, a} \mathbf{c} &= \eta c \textbf{1} \\
    \eta \textbf{1}^T \mathbf{c} &= \eta c
\end{align}
By substituting this into eq.\ref{eq:pos_def_quad}, we get
\begin{equation}
    \eta \mathbf{c}^T \textbf{1} \textbf{1}^T \mathbf{c} - 2 \eta \mathbf{c}^T \textbf{1} \textbf{1}^T \mathbf{c} + \eta (\textbf{1}^T \mathbf{c})^2 = \eta \cdot 0 \ge 0
\end{equation}
The minimum of eq.\ref{eq:pos_def_quad} is zero.
We have shown that the special treatment of the center point guarantees that cylindrical kernels are positive-semidefinite.

\end{proof}

\section{Implementation Detail}
Parts of implementation details is provided in Section 4 except for the prior distribution we use for radius kernel warping eq~8. In order to make BOCK focus more on the center, we make prior concave and non-decreasing by using spike and slab prior~\cite{ishwaran2005spike}.
In eq~8, $\log(\alpha)$ has spike and slab prior on positive real line.
$\log(\beta)$ has spike and slab prior on negative real line.

\section{Benchmark functions}

The suggested search space for below benchmark functions are adjusted to be $[-1, 1]^D$ in our experiments.

\subsection{Repeated Branin}
\begin{equation}
	f_{rep-branin}(x_1, x_2, \cdots, x_D) = 1/\lfloor \frac{D}{2} \rfloor \sum_{i=1}^{\lfloor D/2 \rfloor} f_{branin}(x_{2i-1}, x_{2i})
\end{equation}
where $f_{branin}$ is branin function whose formula can be found in ~\cite{laguna2005experimental}.
The original search space of branin function is $[-5,10] \times [0,15]$

\subsection{Repeated Hartmann6}
\begin{equation}
	f_{rep-hartmann6}(x_1, x_2, \cdots, x_D) = 1/\lfloor \frac{D}{6} \rfloor \sum_{i=1}^{\lfloor D/6 \rfloor} f_{hartmann6}(x_{6i-5}, x_{6i-4}, x_{6i-3}, x_{6i-2}, x_{6i-1}, x_{6i})
\end{equation}
where $f_{hartmann6}$ is hartmann6 function whose formula can be found in ~\cite{laguna2005experimental}.
The original search space of hartmann6 function is $[0,1]^6$

\subsection{Rosenbrock~\cite{laguna2005experimental}}
\begin{equation}
	f_{rosenbrock}(x_1, x_2, \cdots, x_D) = \sum_{i=1}^{D-1} \big[ 100(x_{i+1} - x_i^2)^2 + (x_i-1)^2 \big]
\end{equation}
The original search space is $[-5, 10]^D$

\subsection{Levy~\cite{laguna2005experimental}}
\begin{gather}
	f_{levy}(x_1, x_2, \cdots, x_D) = \sin^2(\pi w_1) \sum_{i=1}^{D-1} (w_i-1)^2 \big[ 1+ 100 \sin^2 (\pi w_i + 1) \big] + (w_D-1)^2\big[1+ \sin^2(2 \pi w_D)\\
    w_i = 1 + \frac{x_i - 1}{4} \nonumber
\end{gather}
The original search space is $[-10, 10]^D$

\section{Efficiency vs accuracy}

We conduct the same analysis for efficiency vs accuracy with other benchmark functions on 20 dimensional case. 
In all cases, BOCK is the closest to the optimum operating point (0, 0)~\ref{figure:runtime_vs_mininum}.
Matern is also accurate enough, although considerably slower,
while SMAC and additive BO are faster but considerably less
accurate. 

\begin{figure}[!htb]
\minipage{0.33\textwidth}
	\centering
    \textbf{Repeated Branin}\par
  	\includegraphics[width=\linewidth]{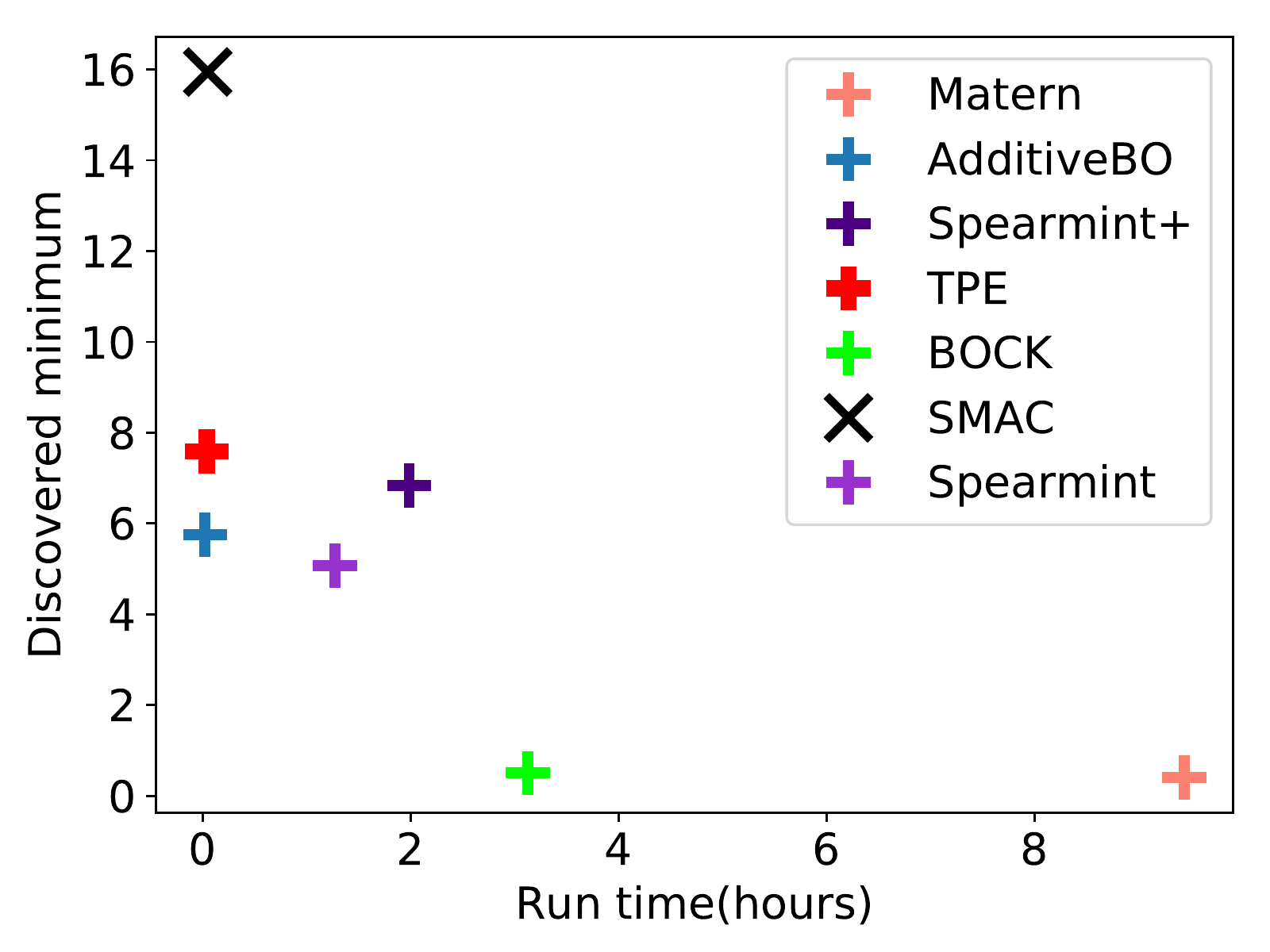}
\endminipage\hfill
\minipage{0.33\textwidth}
\centering
    \textbf{Repeated Hartmann6}\par
  	\includegraphics[width=\linewidth]{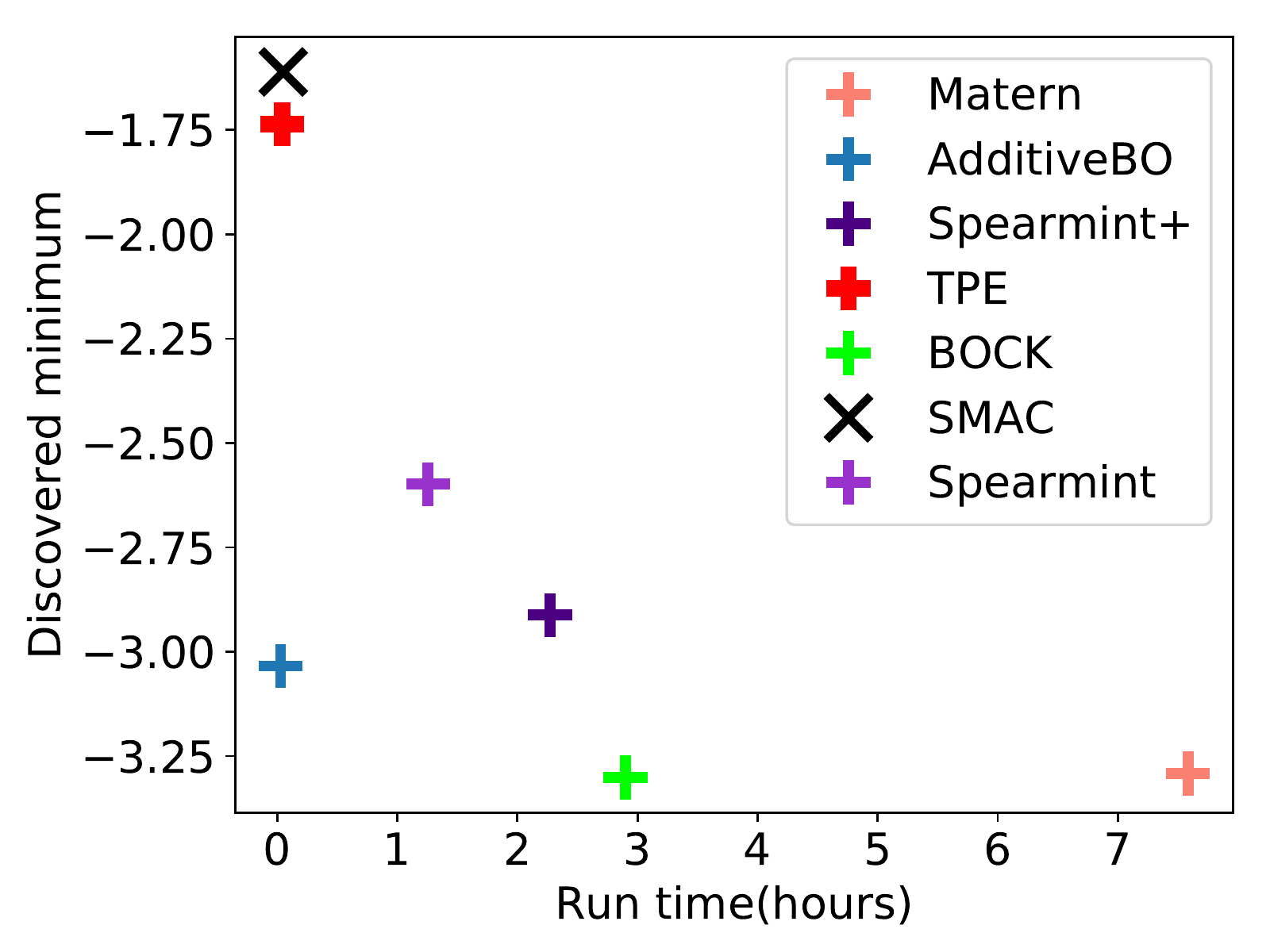}
\endminipage\hfill
\minipage{0.33\textwidth}%
\centering
    \textbf{Levy}\par
  	\includegraphics[width=\linewidth]{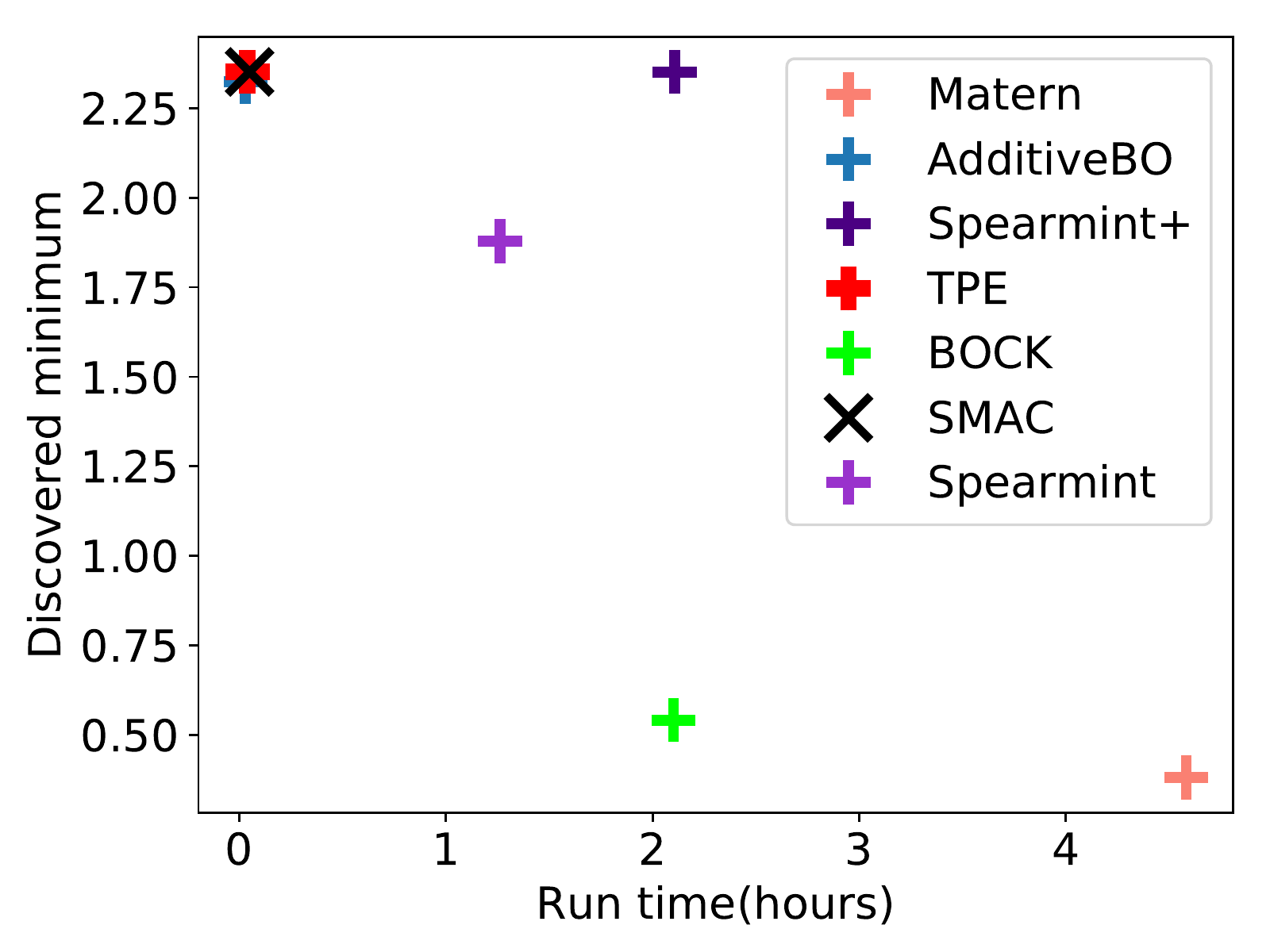}
\endminipage
\caption{
Accuracy of Bayesian Optimization methods vs wall clock time efficiency for the 20-dimensional Repeated Branin, Repeated Hartmann6, Levy benchmark.
BOCK is the closest to the optimum operating point $(0, 0)$. Matern is also accurate enough, although considerably slower, while SMAC and additive BO are faster but considerably less accurate.
}
\label{figure:runtime_vs_mininum}
\end{figure}

\section{Scalability}
We also conduct the experiment to check the scalability of algorithms with other benchmark functions on 20 and 100 dim. The same observation that BOCK is clearly
more efficient and less effected by the increasing dimensionality can be made~\ref{figure:runtime}.

\begin{figure}[!htb]
\minipage{0.33\textwidth}
	\centering
    \textbf{Repeated Branin}\par
  	\includegraphics[width=\linewidth]{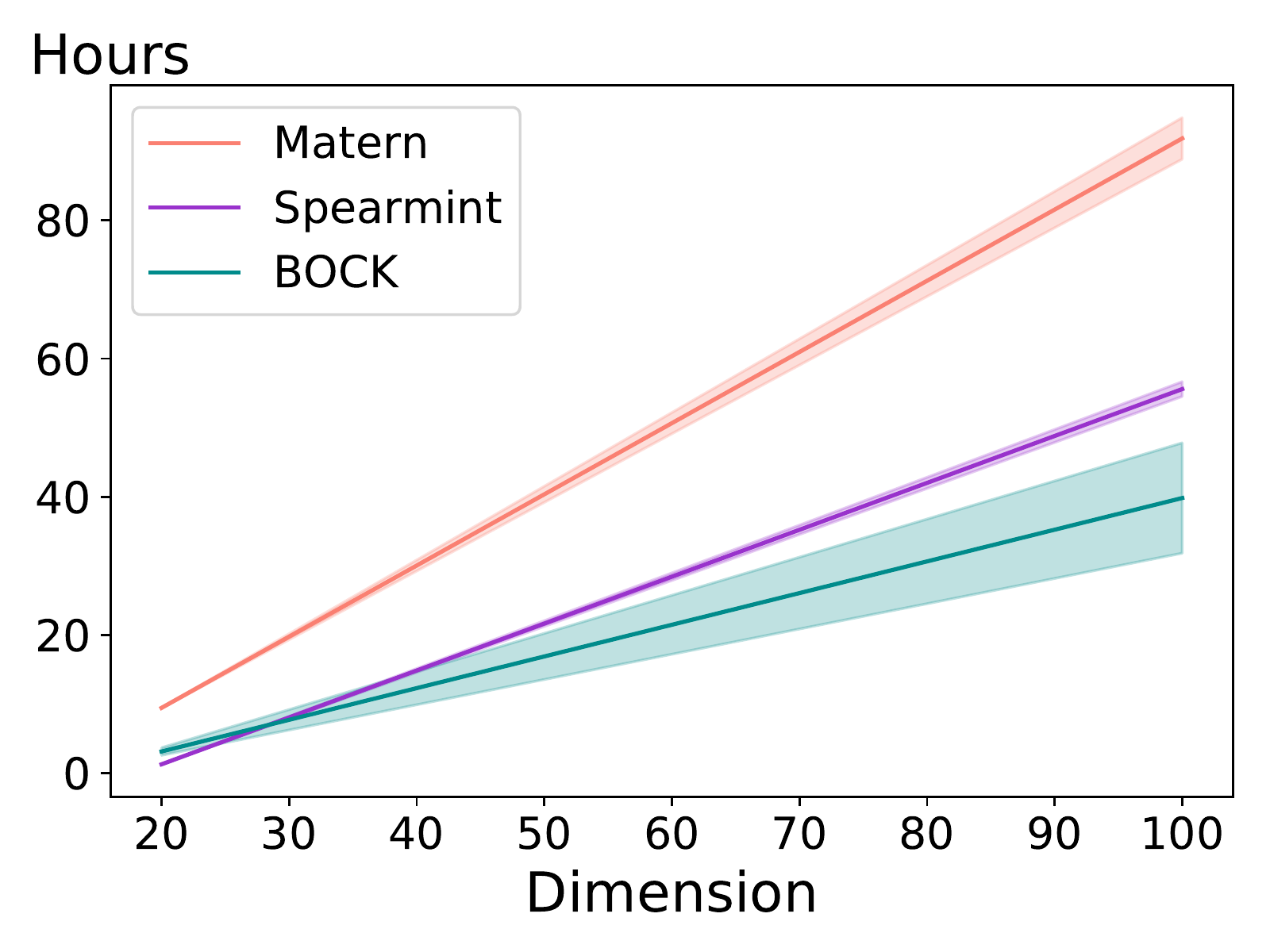}
\endminipage\hfill
\minipage{0.33\textwidth}
\centering
    \textbf{Repeated Hartmann6}\par
  	\includegraphics[width=\linewidth]{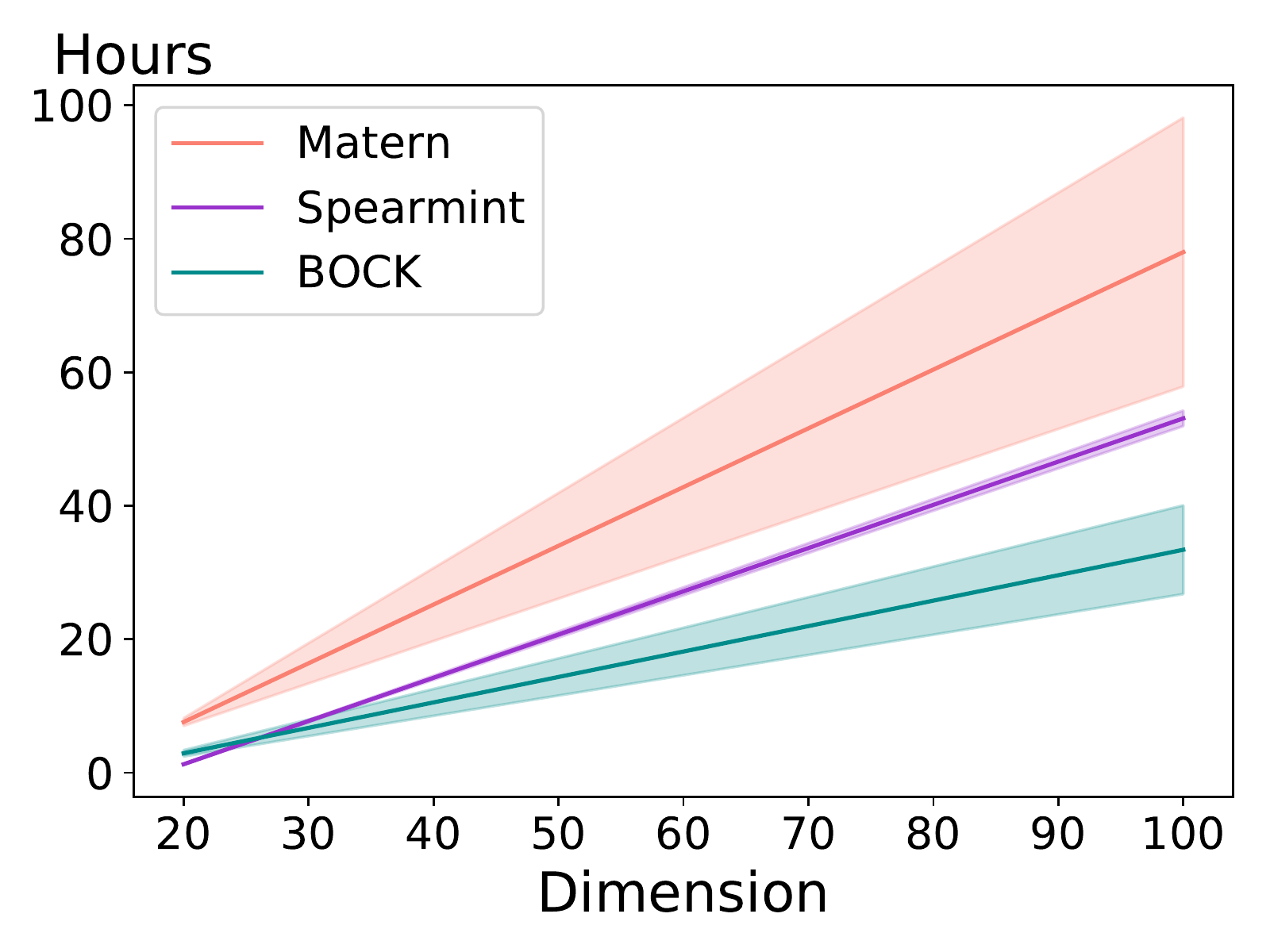}
\endminipage\hfill
\minipage{0.33\textwidth}%
\centering
    \textbf{Levy}\par
  	\includegraphics[width=\linewidth]{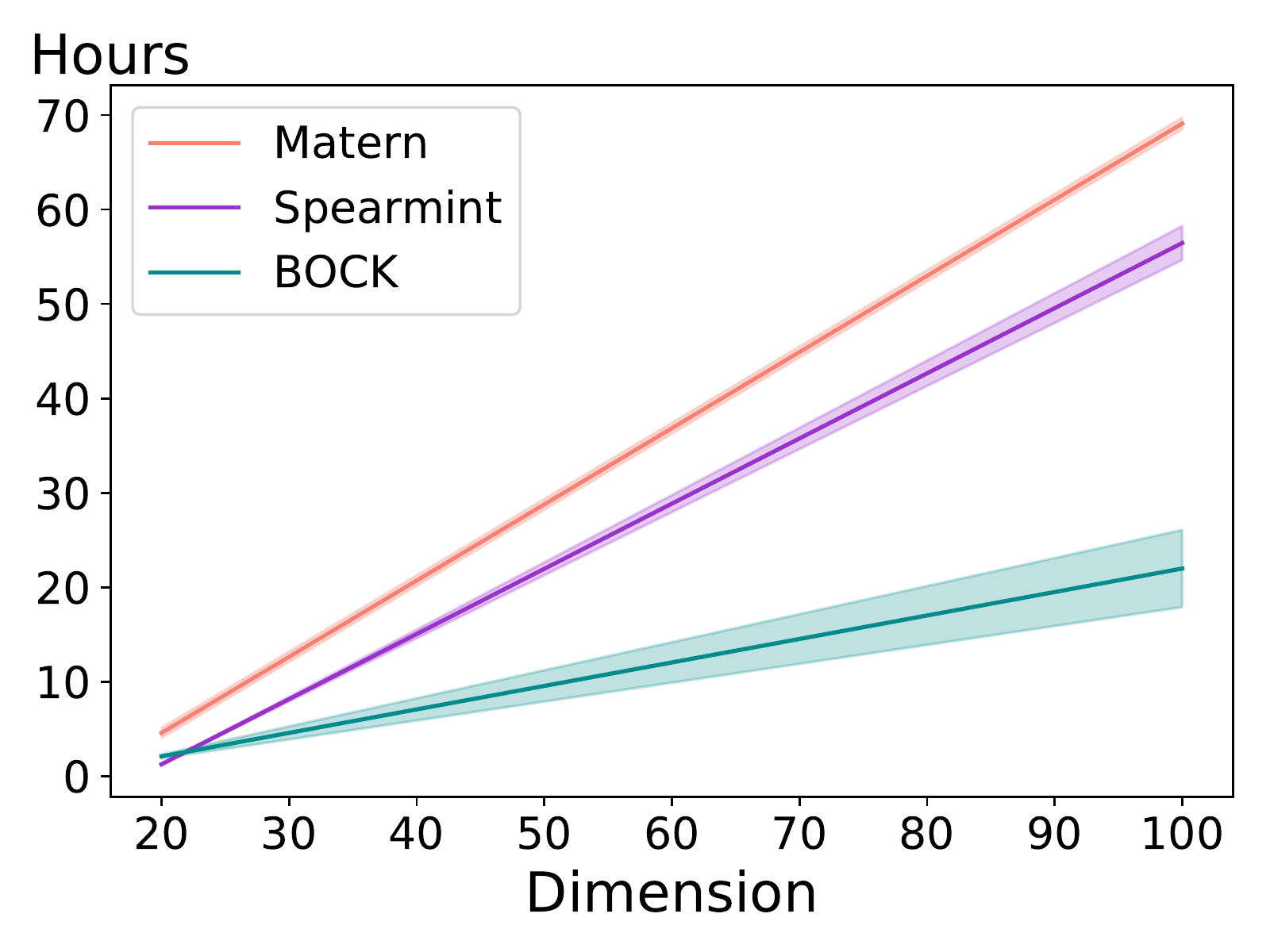}
\endminipage
\caption{
Wall clock time(hours) on the Repeated Branin, Repeated Hartmann6, Levy benchmark for an increasing the number of dimensions (20 and 100 dimensions, using 200 and 600 function evaluations respectively for all methods).
The solid lines and colored regions represent the mean wall clock time and one standard deviation over these 5 runs.
As obtaining the evaluation score $y=f(\x_*)$ on these benchmark functions is instantaneous, the wall clock time is directly related to the computational efficiency of algorithms.
In this figure, we compare BOCK and BOs with relative high accuracy in all benchmark functions, such as Spearmint and Matern.
BOCK is clearly more efficient, all the while being less affected by the increasing number of dimensions.
}
\label{figure:runtime}
\end{figure}

\end{document}